\documentclass[letterpaper, 10 pt, conference]{ieeeconf}  

\IEEEoverridecommandlockouts                              

\usepackage{graphics} 
\usepackage{epsfig} 
\usepackage{times} 
\usepackage{amsmath} 
\usepackage{amssymb}  
\usepackage{textcomp}
\usepackage{stfloats}
\usepackage{url}
\usepackage{verbatim}
\usepackage{graphicx}
\usepackage{mathtools} 
\usepackage{hyperref}
\usepackage{bbm}
\usepackage{bm}
\usepackage{color}
\usepackage{algorithm}
\usepackage{algpseudocode}
\usepackage{balance}
\usepackage{booktabs}
\usepackage{placeins} 
\usepackage{siunitx}
\algnewcommand\algorithmicinput{\textbf{\quad Input:}}
\algnewcommand\INPUT{\item[\algorithmicinput]}
\algnewcommand\algorithmicfind{\qquad Find}
\algnewcommand\FIND{\item[\algorithmicfind]}

\DeclareMathOperator*{\argmin}{arg\,min}

\newtheorem{theorem}{Theorem}

\newtheorem{assumption}{Assumption}
\newtheorem{remark}[theorem]{Remark}

\title{\LARGE \bf
RK-MPC: Residual Koopman Model Predictive Control for Quadruped Locomotion in Offroad Environments
}

\author{Sriram S. K. S. Narayanan and Umesh Vaidya
\thanks{Financial support from NSF CMMI Award 2531804 is greatly acknowledged. Sriram S. K. S. Narayanan and Umesh Vaidya are with the Department of Mechanical Engineering, Clemson University, Clemson, SC 29630, USA
        {\tt\small email:sriramk@clemson.edu; uvaidya@clemson.edu}}%
}

\begin{document}

\maketitle
\thispagestyle{empty}
\pagestyle{empty}

\begin{abstract}
This paper presents Residual Koopman MPC (RK-MPC), a Koopman-based, data-driven model predictive control framework for quadruped locomotion that improves prediction fidelity while preserving real-time tractability. RK-MPC augments a nominal template model with a compact linear residual predictor learned from data in lifted coordinates, enabling systematic correction of model mismatch induced by contact variability and terrain disturbances with provable bounds on multi-step prediction error. The learned residual model is embedded within a convex quadratic-program MPC formulation, yielding a receding-horizon controller that runs onboard at 500\,Hz and retains the structure and constraint-handling advantages of optimization-based control. We evaluate RK-MPC in both Gazebo simulation and Unitree Go1 hardware experiments, demonstrating reliable blind locomotion across contact disturbances, multiple gait schedules, and challenging off-road terrains including grass, gravel, snow, and ice. We further compare against Koopman/EDMD baselines using alternative observable dictionaries, including monomial and $SE(3)$-structured bases, and show that the residual correction improves multi-step prediction and closed-loop performance while reducing sensitivity to the choice of observables. Overall, RK-MPC provides a practical, hardware-validated pathway for data-driven predictive control of quadrupeds in unstructured environments. See \url{https://sriram-2502.github.io/rk-mpc/} for implementation videos.
\end{abstract}

\section{INTRODUCTION}
\subsection{Motivation}
Quadruped robots operating in off-road and unstructured environments must navigate terrain-induced disturbances, contact uncertainty, and modeling errors while satisfying tight real-time constraints \cite{ha2025learning}. Classical physics-based MPC can enforce constraints and provide reliable receding-horizon behavior, but its performance depends on prediction fidelity, which is difficult to maintain across changing footholds, compliance, and unmodeled terrain interactions \cite{wensing2023optimization}. Purely policy-based learning can produce agile behaviors, yet it often requires extensive data and tuning and does not naturally provide predictive structure for constrained navigation \cite{kober2013reinforcement}. These challenges motivate a data-driven, model-centric approach in which predictive models are learned directly from locomotion data and embedded within MPC to preserve optimization-based constraint handling \cite{hewing2020learning}. In this work, we focus on Koopman-based predictors as a structured learning mechanism, enabling the construction of linear predictive models from data that can be directly integrated within MPC frameworks to improve robustness and real-time performance in off-road, unstructured environments.

\subsection{Related works}

Recent work on quadruped locomotion spans a spectrum from physics-based optimization methods to fully data-driven learning approaches. Classical model-based methods formulate locomotion as a trajectory optimization or model predictive control (MPC) problem using reduced-order dynamics, enabling dynamic and reliable behaviors on hardware \cite{winkler2018gait,di2018dynamic}. These approaches leverage structured dynamics and convex formulations to enforce contact and friction constraints, but their performance depends on the fidelity of the underlying model and can degrade under unmodeled terrain interactions, compliance, and disturbances encountered in off-road environments.

To address these limitations, learning-based MPC methods augment nominal physics models with learned components (e.g., residual dynamics or disturbance models), improving prediction accuracy while retaining an optimization-based mechanism for constraint handling \cite{aswani2013provably,deisenroth2013gaussian}. However, for high-dimensional, contact-rich systems such as quadrupeds, these methods can still face practical bottlenecks: collecting sufficiently informative data across gaits, impacts, and terrains is costly; learned corrections can be brittle under distribution shift; and uncertainty-aware variants (e.g., GP-based MPC) introduce additional computational overhead that is difficult to scale in real time \cite{hewing2020learning, hewing2019cautious}.

In parallel, deep reinforcement learning has demonstrated impressive agility and sim-to-real transfer for quadruped locomotion, particularly in unstructured and off-road environments \cite{tan2018sim,hwangbo2019learning,miki2022learning}. These approaches leverage large-scale simulation, domain randomization, and policy learning to achieve robust blind locomotion over challenging terrain. However, such methods are typically highly sample-inefficient, requiring extensive data collection and careful reward design; moreover, learned policies do not inherently provide predictive structure for constrained control or explicit safety guarantees, and can remain sensitive to unmodeled hardware effects and contact variability \cite{ha2025learning,kumar2022adapting}.

\textit{Role of Koopman Operator in structured learning and control:}
Koopman operator theory provides a principled operator-theoretic framework for analyzing nonlinear dynamical systems through linear evolution in an infinite-dimensional function space \cite{koopman1931hamiltonian,mezic2005spectral,mezic2020spectrum,otto2021koopman}. The spectral properties of the Koopman operator reveal intrinsic geometric structures of the state space, including invariant manifolds characterized as level sets of eigenfunctions \cite{budivsic2012applied,mauroy2016global}. This connection between spectral analysis and state-space geometry has enabled a broad class of data-driven methods for approximating Koopman eigenfunctions and constructing lifted representations of nonlinear dynamics. In practice, these lifted representations enable approximating nonlinear systems with linear predictors in the space of observables, enabling the use of linear system identification techniques and control design tools. In particular, Koopman-based predictors have been incorporated into data-driven control pipelines for system identification, state estimation, and model predictive control, where the lifted linear dynamics facilitate tractable optimization and constraint handling \cite{korda2018linear,korda2020optimal}.

Koopman-based methods have also been applied for feedback stabilization and learning structured representations tailored for control synthesis, including approaches based on deep neural network parameterizations \cite{huang2018feedback,han2020deep}. More recent advances have further extended Koopman operator theory to nonlinear control settings, where operator-theoretic formulations enable the construction of control-relevant representations of dynamical systems. In particular, connections between Koopman spectral analysis and nonlinear optimal control have been established through links to the Hamilton–Jacobi equation, demonstrating that solutions to optimal control problems can be approximated using Koopman eigenfunctions within convex optimization frameworks \cite{vaidya2025koopman,vaidya2023data,huang2022convex}. These developments provide a principled pathway for extending linear systems tools to nonlinear dynamics and motivate the use of Koopman-based representations for scalable control design. These ideas have been successfully applied across a range of robotic systems \cite{abraham2019active,shi2026koopman}, including soft manipulators \cite{bruder2020data,haggerty2023control, bruder2025koopman}, autonomous ground vehicles \cite{xiao2022deep,fu2025residual,joglekar2026modeling}, agile quadrotors \cite{narayanan2023se,rajkumar2025real}, and, more recently, legged platforms \cite{krolicki2022modeling,yang2025koopman}, highlighting the growing potential of Koopman-based representations for real-time control in complex, high-dimensional robotic systems.


\textit{Challenges in Extending Koopman Operator methods to quadruped locomotion:} Extending Koopman-based methods to rigid-body legged locomotion remains challenging due to the high-dimensional, contact-rich hybrid dynamics of quadrupeds, where switching contacts, impacts, and terrain variability degrade multi-step predictive accuracy. Recent work has incorporated Koopman predictors into legged systems by learning lifted linear models from locomotion data and embedding them in estimation and predictive control pipelines \cite{yang2025koopman,khorshidi2024centroidal,li2025continual}. Geometry-aware observable dictionaries that encode $SE(3)$ structure improve physical consistency over generic liftings, but do not guarantee preservation of rotation constraints, leading to $SO(3)$ drift and error accumulation over the prediction horizon \cite{narayanan2023se,rajkumar2025real}. Deep lifting approaches can improve one-step accuracy, but require large datasets, introduce significant training overhead, and remain sensitive to distribution shifts across contact conditions and terrains \cite{khorshidi2024centroidal}. Continual lifting strategies mitigate dictionary design issues, but are typically data-intensive and primarily validated in simulation, leaving challenges for reliable deployment on hardware \cite{li2025continual}.

These limitations motivate Koopman predictors for quadruped rigid-body dynamics that (i) retain a compact and computationally light lifted representation suitable for real-time MPC, (ii) reduce sensitivity to observable design while maintaining geometric consistency under multi-step rollouts, and (iii) remain reliable under terrain disturbances and contact variability in unstructured, off-road hardware deployments.

\subsection{Main Contributions}
This paper presents Residual Koopman MPC (RK-MPC), a Koopman-based, data-driven MPC framework with a learned linear residual correction for quadruped locomotion.
Our main contributions are:
\begin{enumerate}
    \item \textit{Residual Koopman modeling:} a Koopman-based data-driven framework that learns a compact linear residual correction on top of a nominal template model, along with formal guarantees on multi-step prediction error through bounded residual dynamics.
    \item \textit{Convex RK-MPC formulation:} a MPC formulation that embeds the learned residual predictor within a receding-horizon controller and runs onboard in real time at 500\,Hz.
    \item \textit{Validation in simulation and hardware:} extensive Gazebo simulations and Unitree Go1 experiments demonstrating reliable tracking of planar velocity commands and robust blind locomotion across contact disturbances, multiple gait schedules, and off-road terrains (grass, gravel, snow, and ice).
\end{enumerate}

\section{PRELIMINARIES}
\subsection{Notation}

For $n\in\mathbb{N}$, $\mathbb{R}^n$ denotes the $n$-dimensional real vector
space. For $m,n\in\mathbb{N}$, $\mathbb{R}^{m\times n}$ denotes the set of real
$m\times n$ matrices, and $\mathbb{I}_n$ denotes the $n\times n$ identity matrix. We denote the state by $x\in\mathcal{M}\subseteq\mathbb{R}^n$ and the control input by $u\in\mathcal{U}\subseteq\mathbb{R}^m$. 

The cross product of two vectors $a,b \in \mathbb{R}^3$ is written as
$a \times b = a^\wedge b$, where $(\cdot)^\wedge : \mathbb{R}^3 \rightarrow
\mathfrak{so}(3)$ denotes the \emph{hat operator}. For
$a = [\,a_1\; a_2\; a_3\,]^\top$, the hat map produces the skew-symmetric matrix
\[
a^\wedge =
\begin{bmatrix}
0 & -a_3 & a_2 \\
a_3 & 0 & -a_1 \\
-a_2 & a_1 & 0
\end{bmatrix}.
\]
The inverse mapping $(\cdot)^\vee : \mathfrak{so}(3) \rightarrow \mathbb{R}^3$
denotes the \emph{vee operator}. The Lie group of rotation matrices is denoted
$SO(3)$, with associated Lie algebra $\mathfrak{so}(3)$.

For a matrix $A \in \mathbb{R}^{m \times n}$, $\mathrm{vec}(A)$ denotes its
column-wise vectorization, $\|\cdot\|_F$ denotes the Frobenius norm, and
$(\cdot)^\dagger$ denotes the Moore--Penrose pseudoinverse. We use $\Theta \in \mathbb{R}^3$ to denote the base (center-of-mass) orientation parameterized by roll--pitch--yaw (RPY) Euler angles, i.e., $\Theta:= [\theta_r\;\theta_p\;\theta_y]^\top$, where $\theta_r$, $\theta_p$ and $\theta_y$ denote the roll, pitch and yaw angle of the center of mass respectively.

\subsection{Single Rigid Body (SRB) Dynamics} \label{sec:srbd}
\begin{figure}
    \centering
    \includegraphics[width=0.75\linewidth]{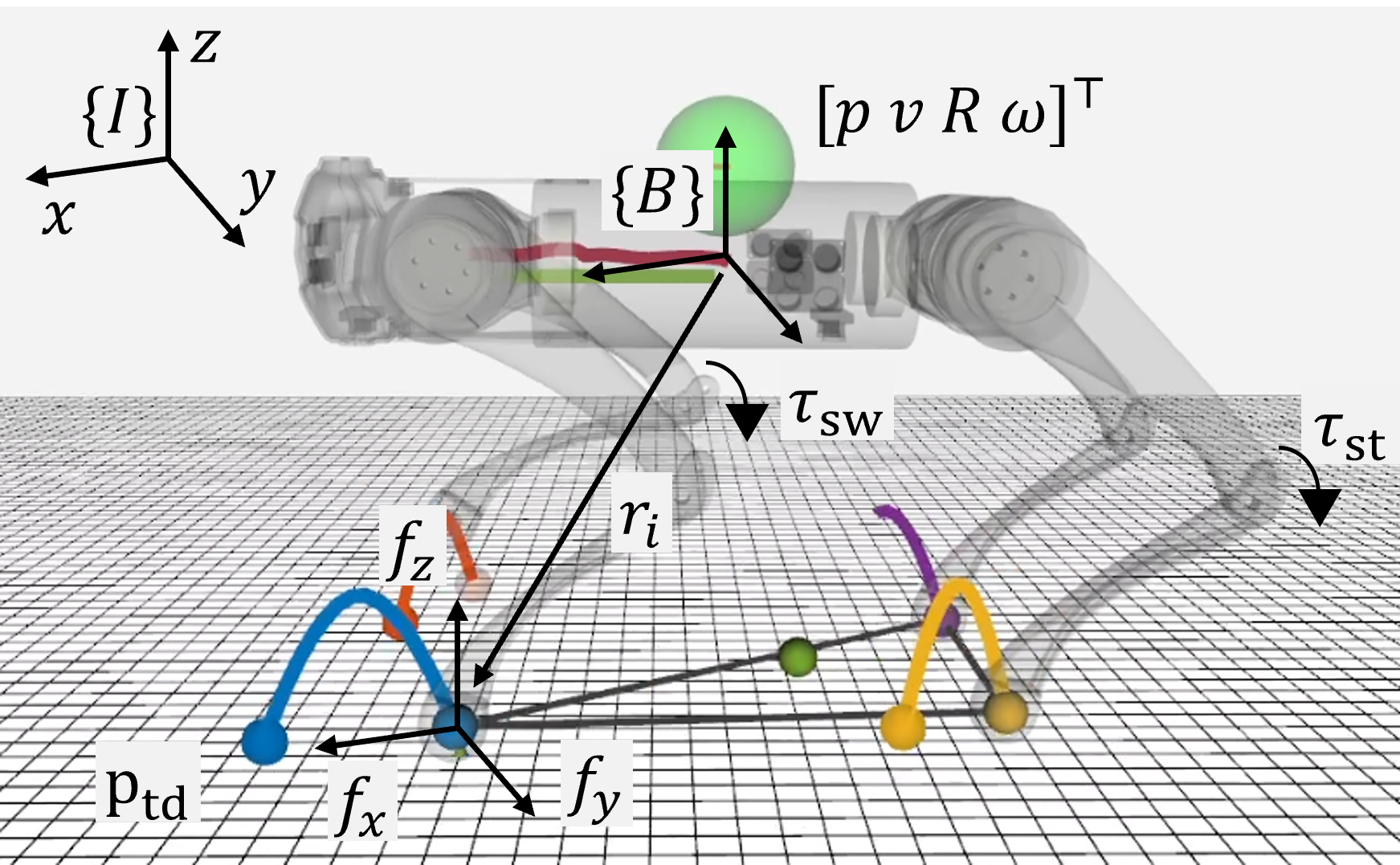}
    \caption{\textbf{Single rigid-body (SRB) model.}
The floating base evolves under stance ground reaction forces $f_i$ applied at the contact points with moment arms $r_i$, generating the net centroidal wrench that propels the robot forward.}

    \label{fig:rigid_body}
\end{figure}
Quadruped locomotion is modeled as a hybrid system switching between swing and
stance phases. During stance, the floating base is underactuated and its motion
is determined solely by the ground-reaction forces applied at the
feet contact. Each foot $i$ applies a force $f_i \in \mathbb{R}^3$ at a location $r_i$
relative to the CoM (for $i = 1, \hdots, n_c$, where $n_c$ is the number of feet in contact). The resulting net wrench can be obtained using the transformation matrix $H$ as follows
\begin{equation}
\begin{bmatrix} F \\[1mm] \tau \end{bmatrix}
=
\underbrace{
\begin{bmatrix}
\mathbb{I}_3 & \hdots & \mathbb{I}_3 \\
r_1^\wedge & \hdots & r_{n_c}^\wedge
\end{bmatrix}
}_{H}
\begin{bmatrix} f_1 \\ \vdots \\ f_{n_c} \end{bmatrix}, \label{eq:force_map}
\end{equation}
where $\mathbb{I}_3$ is the $3 \times 3$ identity matrix. The evolution of the floating-base can be described using the single rigid body
(SRB) dynamics model \cite{winkler18_phd,orin2013centroidal}:

\begin{subequations}\label{eq:srb_dynamics}
\begin{align}
    \dot{{p}} &= {v}, \\
    \dot{{v}} &= {a}_g + \frac{1}{m}{F}, \\
    \dot{{R}} &= {\omega}^\wedge {R}, \label{eq:nonlin_angles} \\
    \dot{{\omega}} &=
    {I}^{-1}\!\Big({\tau} - {\omega}^\wedge {I} {\omega}\Big), \label{eq:nonlin_ang_mom} \\ 
    {I} &:= {R} {I}_B {R}^\top
\end{align}
\end{subequations}
where $p \in \mathbb{R}^3$ and $v \in \mathbb{R}^3$ denote the CoM position and
velocity in the inertial frame $\mathcal{I}$, $R \in SO(3)$ is the body
orientation, and $\omega \in \mathbb{R}^3$ is the angular velocity expressed in
the body frame $\mathcal{B}$, $m$ is the body mass, ${a}_g=[0\;0\;-g]^\top$ is gravity expressed in
the inertial frame, ${I}\in\mathbb{R}^{3\times3}$ is the Inertia expressed in the inertial frame and ${I}_B\in\mathbb{R}^{3\times3}$ is the inertia
matrix expressed in the body frame (see Fig. \ref{fig:rigid_body}).

\begin{table}[t]
\centering
\caption{SRB physical parameters for Unitree Go1}
\label{tab:srb_params}
\footnotesize
\setlength{\tabcolsep}{2.6pt}
\renewcommand{\arraystretch}{0.90}
\begin{tabular}{lc}
\toprule
Parameter & Value \\
\midrule
Mass $m$ [kg] & 12.75 \\
Inertia $I_B$ [kg\,m$^2$] &
$10^{-3}\!\times\!
\begin{bmatrix}
160 & 0.12 & -16\\
0.12 & 470 & -0.03\\
-16 & -0.03 & 520
\end{bmatrix}$ \\
\bottomrule
\end{tabular}
\end{table}

\subsection{Optimal Control Formulation}
The finite-horizon optimal control problem in~\eqref{eq:ocp_general} seeks a sequence of states and inputs 
$\{x_i,u_i\}_{i=0}^N$ that minimizes a cumulative performance index while satisfying the system dynamics and all admissible constraints. Such formulations are standard in model predictive control and optimal control of constrained dynamical systems
\cite{rawlings2017model}. In this paper, we consider the optimal control problem in the following form:
\begin{subequations}\label{eq:ocp_general}
\begin{align}
\min_{\{x_k,u_k\}_{k=0}^N} \quad 
& \ell_{N+1}(x_{N+1})
  + \sum_{k=0}^{N} \ell_i(x_k, u_k)
\label{eq:ocp_general_obj} \\[4pt]
\text{s.t.} \quad
x_{k+1} &= F(x_k, u_k), 
\label{eq:ocp_general_dyn} \\[4pt]
 u_k &\in \mathcal{U}_k,
\label{eq:ocp_general_u} \\[4pt]
 x_k &\in \mathcal{X}_k,
\label{eq:ocp_general_x} \\[4pt]
 x_0 &= x_{\mathrm{op}}.
\label{eq:ocp_general_init}
\end{align}
\end{subequations}
 
The cost function is composed of a terminal cost $\ell_{N+1}(x_{N+1})$ and stage costs $\ell_i(x_k,u_k)$, which together encode tracking objectives and control effort penalties. The discrete-time dynamics $x_{k+1} = F(x_k,u_k)$ enforce the evolution of the system across the horizon, where $F(\cdot)$ denotes the transition map. At each stage $k$, the control input must lie within the admissible set $\mathcal{U}_k$, and the state must remain within the feasible set $\mathcal{X}_k$. The OCP is initialized at the measured or estimated operating state $x_0 = x_{\mathrm{op}}$. 

The constraints in~\eqref{eq:ocp_general} take on a specific structure for legged locomotion. The system dynamics $F(x_k, u_k)$ in \eqref{eq:ocp_general_dyn} describe the discrete-time evolution of the floating-base state using rigid-body dynamics. The control inputs $u_i$ correspond to the ground reaction forces applied at the feet $f_i \in \mathbb{R}^3$, which generate linear and angular accelerations of the floating base. The control objective is to design these foot forces to track a desired CoM reference trajectory while satisfying friction and contact constraints. \cite{wensing2023optimization,di2018dynamic,winkler18_phd,focchi2017high}.

The admissible input set $\mathcal{U}_k$ in \eqref{eq:ocp_general_u} corresponds to the feasible ground reaction forces at each foot. In general, these forces must lie within the Coulomb friction cone, which imposes the nonlinear constraint
\[
\sqrt{f_{i,x}^2 + f_{i,y}^2} \,\le\, \mu f_{i,z}, 
\qquad f_{i,z} \ge 0,
\]
for each stance foot~$i$, where $\mu$ is the coefficient of friction. 

Finally, we do not impose additional state constraints beyond those required by the dynamics, and thus take $\mathcal{X}_i = \mathbb{R}^n$ in~\eqref{eq:ocp_general_x}. This will be explored in future works.

\begin{figure*}
  \centering
  \includegraphics[width=0.75\textwidth]{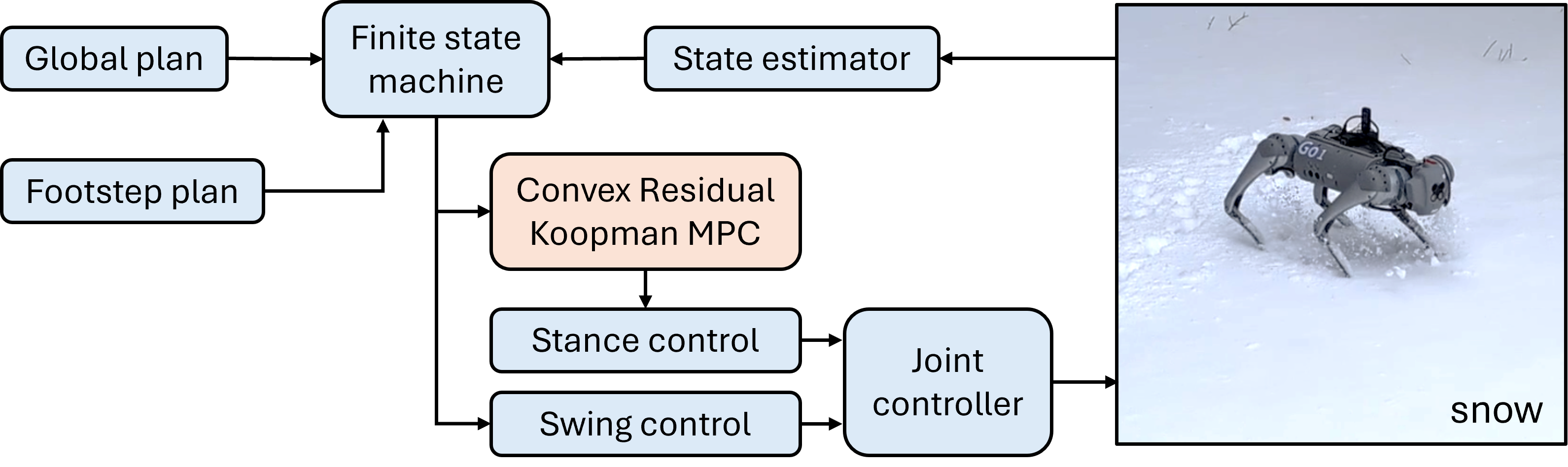}
  \caption{\textbf{RK-MPC locomotion stack.} A hierarchical pipeline converts high-level velocity commands into joint-level torque/position commands. A state estimator provides the feedback state to the finite-state machine and the convex residual Koopman MPC, which generates optimal ground reaction forces. The resulting commands are tracked by the joint controller and executed on hardware.}
  \label{fig:rk_mpc_stack}
\end{figure*}

\subsection{Hierarchical Control Architecture}

The quadruped control system is organized as a hierarchical architecture that
separates high-level motion planning from low-level actuation. This structure
enables real-time implementation while preserving modularity between planning,
optimization, and joint-level control \cite{di2018dynamic, winkler18_phd}. Note that this architecture is for blind locomotion, where no exteroceptive sensing is used (see Fig. \ref{fig:rk_mpc_stack}).

\textbf{{Finite State Machine.}}
A finite state machine (FSM) coordinates transitions between swing and stance
phases for each leg based on the gait schedule and contact state. The FSM
activates the appropriate controller for each leg and ensures consistent timing
between planning, optimization, and execution layers.

\textbf{Global Planner.}
A reduced-order model of the floating base is used to generate reference trajectories for the center of mass and base orientation. These references are
specified in terms of desired linear and angular velocities and serve as inputs
to the lower-level controllers.

\textbf{Footstep Planner.}
Footstep planning produces touchdown targets for swing feet. In this work, we
use a kinematic heuristic that accounts for gait timing and commanded motion.
Let $\phi_i\in[0,1]$ denote the swing phase for leg $i\in\bar{C}$, where \(\bar{C}\) denotes the set of swing legs. Let
$T_{\mathrm{sw}}$ and $T_{\mathrm{st}}$ denote the swing and stance durations.
Let $(v_x,v_y)$ and $\omega_z$ denote the estimated planar body velocity and yaw rate, and let $(v_x^{\mathrm{cmd}},v_y^{\mathrm{cmd}},\omega_z^{\mathrm{cmd}})$
denote the commanded twist. The step offsets are defined as
\begin{align}
\Delta x_i &= v_x\!\left[(1-\phi_i)T_{\mathrm{sw}} + \tfrac{1}{2}T_{\mathrm{st}}\right]
          + k_x \left(v_x - v_x^{\mathrm{cmd}}\right), \\
\Delta y_i &= v_y\!\left[(1-\phi_i)T_{\mathrm{sw}} + \tfrac{1}{2}T_{\mathrm{st}}\right]
          + k_y \left(v_y - v_y^{\mathrm{cmd}}\right), \\
\Delta \psi_i &= \omega_z\!\left[(1-\phi_i)T_{\mathrm{sw}} + \tfrac{1}{2}T_{\mathrm{st}}\right]
          + k_\psi \left(\omega_z^{\mathrm{cmd}}-\omega_z\right).
\end{align}
The gains $(k_x,k_y,k_\psi)$ are feedback gains that regulate velocity and yaw tracking by adjusting foot placement. The touchdown location for leg $i$ is then given by
\begin{equation}
{p}^{(i)}_{\text{td}} =
{p}_{\text{body}} +
\begin{bmatrix}\Delta x_i\\ \Delta y_i\\ 0\end{bmatrix} +
\begin{bmatrix}
r_i\cos(\theta_y+\alpha_i+\Delta\psi_i)\\
r_i\sin(\theta_y+\alpha_i+\Delta\psi_i)\\
0
\end{bmatrix},
\label{eq:raibert_td}
\end{equation}
where $\theta_y$ denotes the base yaw angle, and $(r_i,\alpha_i)$ are the nominal hip-to-foot polar coordinates in the horizontal plane \cite{raibert1986legged}. The touchdown height is set to the nominal ground plane. In implementation, $(\Delta x_i,\Delta y_i)$ are saturated to satisfy kinematic limits (Table~\ref{tab:gait_params}), improving robustness.

\begin{table}[t]
\centering
\caption{Gait and footstep parameters.}
\label{tab:gait_params}
\footnotesize
\setlength{\tabcolsep}{2.6pt}
\renewcommand{\arraystretch}{0.90}
\begin{tabular}{lcc}
\toprule
Parameter & Trot & Crawl \\
\midrule
$T$ [s] & 0.45 & 1.10 \\
$\beta$ & 0.5 & 0.75 \\
$T_{\text{st}}$ [s] & 0.225 & 0.825 \\
$T_{\text{sw}}$ [s] & 0.225 & 0.275 \\
Freq. [Hz] & 2.22 & 0.91 \\
Phase bias (LF,LH,RF,RH) & [0,0.5,0.5,0] & [0,0.25,0.5,0.75] \\
$h$ [m] & 0.10 & 0.10 \\
$(k_x,k_y,k_\psi)$ & (0.005,0.005,0.005) & same \\
$(|\Delta x|,|\Delta y|)$ [m] & (0.10,0.10) & same \\
\bottomrule
\end{tabular}
\end{table}

\textbf{Swing-Leg Control.}
A smooth swing trajectory connects the current foot position $(x_s,y_s,z_s)$ to the touchdown target $(x_e,y_e,z_e)$ using the swing phase $\phi\in[0,1]$. We use a cycloidal horizontal profile and a sinusoidal vertical profile with swing height $h$:
\begin{align*}
x(\phi) &= x_s + (x_e-x_s)\frac{2\pi\phi-\sin(2\pi\phi)}{2\pi},\\
y(\phi) &= y_s + (y_e-y_s)\phi,\\
z(\phi) &= z_s + h\frac{1-\cos(2\pi\phi)}{2}.
\end{align*}
Let ${p}_{\textrm{foot}}^{(i)}$ and $\dot{{p}}_{\textrm{foot}}^{(i)}$ denote the measured swing-foot position and velocity, and ${p}_{\textrm{foot},\star}^{(i)}$, $\dot{{p}}_{\textrm{foot},\star}^{(i)}$ the corresponding swing references. The desired Cartesian swing force is
\begin{equation}
{F}^{(i)}_{\text{sw}} =
{K}^{\text{sw}}_p\!\left({p}_{\textrm{foot},\star}^{(i)}-{p}_{\textrm{foot}}^{(i)}\right) +
{K}^{\text{sw}}_d\!\left(\dot{{p}}_{\textrm{foot},\star}^{(i)}-\dot{{p}}_{\textrm{foot}}^{(i)}\right).
\end{equation}
and the joint torques are computed via the Jacobian transpose map
\begin{equation}
{\tau}^{(i)}_{\text{sw}} = {J}_i^\top {F}^{(i)}_{\text{sw}}.
\label{eq:tau_swing}
\end{equation}
Here, ${J}_i$ denotes the leg Jacobian consistent with the Cartesian force representation; ${F}^{(i)}_{\text{sw}}$ is expressed in the same frame as ${J}_i$. Joint-level PD gains are applied to stabilize tracking and ensure bounded torques (Table~\ref{tab:torque_params}).

\textbf{Stance-Leg Control.}
During stance, the high-level MPC computes the desired per-foot ground reaction forces ${f}_i$ for each stance foot. These forces are mapped to joint torques using
\begin{equation}
{\tau}^{(i)}_{\text{st}} = {J}_i^\top {f}_i .
\label{eq:tau_stance}
\end{equation}
The contact forces are expressed in the same frame as the Jacobian used in \eqref{eq:tau_stance}. Torque saturation is applied to respect actuator limits, and the finite-state machine switches between \eqref{eq:tau_swing} and \eqref{eq:tau_stance} according to the gait schedule and contact state.

\begin{table}[t]
\centering
\caption{Torque control parameters.}
\label{tab:torque_params}
\footnotesize
\setlength{\tabcolsep}{3.0pt}
\renewcommand{\arraystretch}{0.92}
\begin{tabular}{lll}
\toprule
Parameter name & Symbol & Value \\
\midrule
Swing Cartesian P gain & $K^{\mathrm{sw}}_{p}$ & $\mathrm{diag}(400,\,400,\,400)$ \\
Swing Cartesian D gain  & $K^{\mathrm{sw}}_{d}$ & $\mathrm{diag}(10,\,10,\,10)$ \\
Joint PD gains (swing)  & $(k_p^{\mathrm{sw}},\,k_d^{\mathrm{sw}})$ & $(3,\,2)$ \\
Joint PD gains (stance) & $(k_p^{\mathrm{st}},\,k_d^{\mathrm{st}})$ & $(0.8,\,0.8)$ \\
Torque saturation limit [Nm] & $\tau_{\max}$ & $50$ \\
\bottomrule
\end{tabular}
\end{table}

\textbf{State Estimation.}
A state estimator provides the floating-base SRB state used throughout the
hierarchy. At each control cycle, onboard proprioceptive measurements (IMU,
leg kinematics, and contact signals from force sensors on each foot) are used to estimate the base pose and body
twist. Unless otherwise stated, all estimated quantities are expressed in the
inertial/world frame $\mathcal{I}$. This estimate is used by the global planner
and Raibert footstep heuristic to construct reference trajectories and touchdown
locations, and in the MPC.

\subsection{Koopman Operator Theory}
\label{sec:koopman_prelims}
Koopman operator theory provides a linear (but generally infinite-dimensional)
representation of nonlinear dynamical systems by studying the evolution of
observable functions along trajectories~\cite{mezic2005spectral,mezic2013analysis}. Consider the discrete-time dynamical system
\begin{equation}
x_{k+1}=T(x_k), \qquad x_k\in\mathcal{M},
\label{eq:autonomous_dt}
\end{equation}
where $\mathcal{M}\subseteq\mathbb{R}^n$ is the state space and
$T:\mathcal{M}\to\mathcal{M}$ is a (possibly nonlinear) state-transition map.
Let $\mathcal{H}_{\mathcal{M}}$ denote a linear space of real-valued observables
$\phi:\mathcal{M}\to\mathbb{R}$. The Koopman operator
$\mathbb{U}:\mathcal{H}_{\mathcal{M}}\to\mathcal{H}_{\mathcal{M}}$ associated
with \eqref{eq:autonomous_dt} is defined by
\begin{equation}
(\mathbb{U}\phi)(x_k) := \phi\!\left(T(x_k)\right) = \phi(x_{k+1}).
\end{equation}
The Koopman operator can be extended to controlled systems by augmenting the
state with the full input sequence~\cite{korda2018linear,williams2016extending,proctor2016dynamic}.
Consider the controlled discrete-time system
\begin{equation}
x_{k+1}=T(x_k,u_k), \qquad x_k\in\mathcal{M},\ \ u_k\in\mathcal{U}.
\label{eq:controlled_dt}
\end{equation}
Let $\ell(\mathcal{U}) := \{(u_i)_{i=0}^{\infty}\mid u_i\in\mathcal{U}\}$ denote
the space of input sequences and define the extended state space
$\mathcal{S}:=\mathcal{M}\times \ell(\mathcal{U})$. Let $\mathcal{H}_{\mathcal{S}}$
denote a linear space of real-valued observables $\Phi:\mathcal{S}\to\mathbb{R}$.
The associated Koopman operator
$\mathcal{K}:\mathcal{H}_{\mathcal{S}}\to\mathcal{H}_{\mathcal{S}}$ is defined by
\begin{equation}
(\mathcal{K}\Phi)\!\left(x_k,(u_i)_{i=k}^{\infty}\right)
:=
\Phi\!\left(x_{k+1},(u_i)_{i=k+1}^{\infty}\right),
\end{equation}
where $x_{k+1}=T(x_k,u_k)$ and $(u_i)_{i=k+1}^{\infty}\in\ell(\mathcal{U})$ denotes
the shifted input sequence. In the following section, we describe how to obtain
a finite-dimensional approximation of $\mathcal{K}$ using EDMD with control.

\subsection{EDMD with Control}
\label{sec:EDMD}

We use EDMD with control (EDMDc) to obtain a finite-dimensional approximation of
the controlled Koopman operator $\mathcal{K}$ from
Section~\ref{sec:koopman_prelims} by projecting onto a subspace spanned by a
dictionary of observables~\cite{korda2018linear,proctor2016dynamic}. Let
$\{\psi_i\}_{i=1}^{q}$ be scalar observables $\psi_i:\mathcal{M}\to\mathbb{R}$ and
define the lifted (vector-valued) observable
\[
\psi(x_k) := \begin{bmatrix}\psi_1(x_k) & \cdots & \psi_q(x_k)\end{bmatrix}^\top \in \mathbb{R}^{q}.
\]
In particular, $K$ is parameterized by a pair of matrices $A\in\mathbb{R}^{q\times q}$ and $B\in\mathbb{R}^{q\times m}$, i.e., $K := \begin{bmatrix} A & B \end{bmatrix}$ which define a linear predictor in the lifted space.
\begin{equation}
\psi(x_{k+1}) \approx A\,\psi(x_k) + B\,u_k .
\label{eq:edmdc_lifted_dyn}
\end{equation}
Given snapshot data $\{(x_k,u_k,x_{k+1})\}_{k=0}^{M-1}$, we first form the lifted snapshot
matrices
\begin{align*}
Z   &:= \begin{bmatrix}\psi(x_0) & \cdots & \psi(x_{M-1})\end{bmatrix},\\
Z'  &:= \begin{bmatrix}\psi(x_1) & \cdots & \psi(x_{M})\end{bmatrix},\\
U   &:= \begin{bmatrix}u_0 & \cdots & u_{M-1}\end{bmatrix}.
\end{align*}
and the regressors as $\Omega := \begin{bmatrix} Z \\ U \end{bmatrix}$. We
estimate $K$ via ridge-regularized least squares,
\begin{equation}
K
=\argmin_{K}\ \|Z'- K\,\Omega\|_F^2+\lambda\|K\|_F^2,
\label{eq:edmdc_ridge}
\end{equation}
with regularization weight $\lambda>0$. This admits the closed-form solution
\begin{equation}
K = Z'\,\Omega^\top\big(\Omega\,\Omega^\top+\lambda \mathbb{I}\big)^{-1},
\label{eq:edmdc_ridge_solution}
\end{equation}
where $\mathbb{I}$ is the identity matrix of appropriate dimension. The resulting predictor~\eqref{eq:edmdc_lifted_dyn} yields a linear lifted model
amenable to fast, optimization-based control (e.g., linear MPC) and has been
successfully used in several robotics settings \cite{shi2026koopman,abraham2019active}. However, the practical performance of EDMDc hinges on the choice of observables, which remains problem-dependent and nontrivial. The following section introduces our residual Koopman framework to reduce this burden while retaining a control-oriented linear structure.


\section{Residual Koopman Modeling Framework}
\label{sec:residual_koopman}

\begin{figure*}
    \centering
    \includegraphics[width=0.8\linewidth]{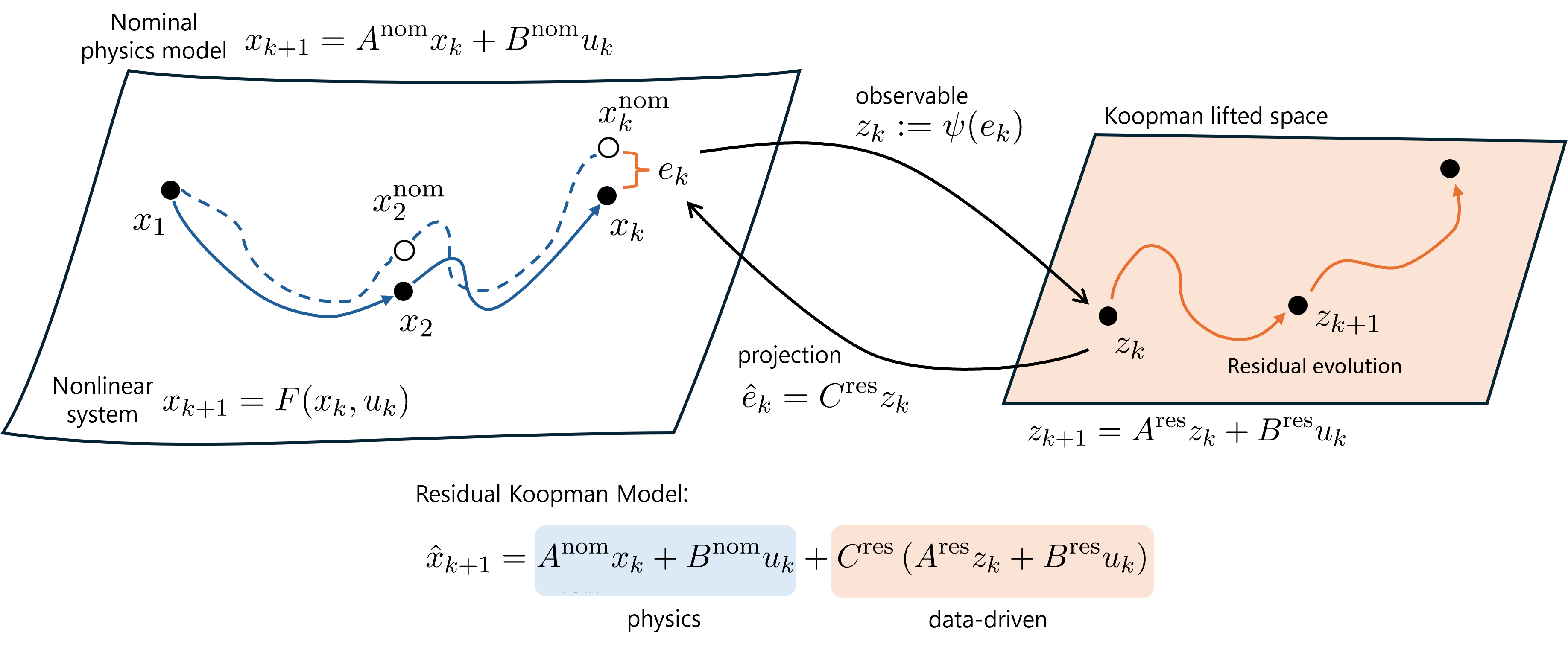}
    \caption{Residual Koopman modeling framework}
    \label{fig:residual_model}
\end{figure*}


This section presents a residual Koopman modeling framework for augmenting a
nominal dynamics model with compact, data-driven corrections while preserving a
linear predictor structure suitable for MPC. The key idea is to (i) begin with a
nominal model that captures the dominant template dynamics, and (ii) learn a
Koopman based linear residual model to account for systematic mismatch
arising from unmodeled contact effects, terrain/friction variation, actuator
dynamics, and linearization errors (see Fig. \ref{fig:residual_model}).

\subsection{Nominal Dynamics Model}
\label{subsec:nominal_model}

Let $x_k, \,x_k^{\mathrm{nom}} \in \mathbb{R}^n$ denote the system state and the nominal predicted state, respectively, and $u_k \in \mathbb{R}^m$ the
control input at time step $k$. We assume that a nominal predictor $F^{\mathrm{nom}}(x_k,u_k)$ is available in discrete time,
\begin{equation}
\label{eq:nominal_general}
x_{k+1}^{\mathrm{nom}} = F^{\mathrm{nom}}(x_k,u_k).
\end{equation}
This predictor captures the primary structure of the true dynamics but may be inaccurate outside a reference regime. The nominal model \eqref{eq:nominal_general} can be obtained in multiple ways. For instance, it may be derived from first-principles physics and subsequently embedded into a Koopman-linear representation through an analytic generator-based construction \cite{bruder2020data, bruder2025koopman}.

In this work, we instantiate $F^{\mathrm{nom}}$ using a discrete linear time-varying
SRB model parameterized by the yaw angle and nominal stance foot locations. We consider the SRB state vector (expressed in the inertial frame)
\begin{equation}
\label{eqn:original_states}
x := [\,p^\top\; \Theta^\top\; v^\top\; \omega^\top\; 1\,]^\top \nonumber
\in \mathbb{R}^{13},
\end{equation}
where $\Theta\in\mathbb{R}^3$ denotes the base orientation (using Euler angles), and the appended constant state enables contributions (e.g., gravity) to be represented within a linear state-space form without an explicit affine offset. The linearized SRB dynamics are given by:
\begin{equation}
\label{eq:nominal_ltv}
x_{k+1}^{\mathrm{nom}}
=
A^{\mathrm{nom}}\, x_k
+
B^{\mathrm{nom}}\, u_k,
\end{equation}
where $A^{\mathrm{nom}} := A(\theta_k)$, $B^{\mathrm{nom}} := B(\theta_k,\bar r_k)$, $\theta_k$ denotes the yaw angle of the base, and $\bar r_k := [r_{1,k}^\top\ \cdots\ r_{n,k}^\top]^\top$ collects the nominal contact locations (moment arms) of the $n$ feet in contact at step $k$. 


\begin{remark}
Alternatively, the nominal predictor can be constructed using an analytical Koopman lifting based on an ${SE}(3)$-consistent basis, where rigid-body kinematics and dynamics are embedded into a linear representation in a higher-dimensional space \cite{narayanan2023se, chen2022koopman}. Such constructions provide a structured and physically consistent nominal model for quadruped locomotion, and can be used in place of the SRB-based linearization without loss of generality.
\end{remark}

The linearized SRB template has been shown to perform competitively with linear predictors
identified directly from locomotion data and therefore
forms a strong starting point for our nominal model \cite{fawcett2022toward}. Despite its practicality, the SRB linearization has well-known limitations. It is most accurate for near-planar locomotion with relatively small roll/pitch
excursions and moderate angular rates \cite{winkler18_phd}. In particular, the rotational kinematics
and dynamics are simplified to obtain a linear (yaw-scheduled) template model.
First, assuming small roll/pitch angles, the nonlinear Euler-angle kinematics are
approximated by the yaw-only mapping
\begin{equation}
\label{eq:euler_rate_approx}
\begin{bmatrix}\dot\phi\\ \dot\theta\\ \dot\psi\end{bmatrix} \nonumber
\approx R_z^\top(\psi)\,\omega,
\end{equation}
where $R_z(\psi)$ denotes the rotation matrix about the $z$-axis. This
approximation becomes inaccurate when roll/pitch excursions are large (e.g.,
highly uneven terrain or aggressive body reorientation), since the exact mapping
between $\dot{\phi},\dot{\theta},\dot{\psi}$ and $\omega$ is nonlinear  \eqref{eq:nonlin_angles}.

Second, writing the body-frame angular momentum as $L = I\omega$, we neglect the
gyroscopic coupling in the momentum rate,
\begin{equation}
\label{eq:angmom_approx}
\frac{d}{dt}(I\omega) = I\dot\omega + \omega \times (I\omega) 
\;\approx\; I\dot\omega, \nonumber
\end{equation}
thereby removing the quadratic coupling between angular velocity and angular
momentum \eqref{eq:nonlin_ang_mom}. This simplification is reasonable for moderate angular rates but can
lead to systematic prediction error in regimes with large angular momentum
(e.g., aggressive maneuvers, highly uneven terrain), motivating the residual Koopman correction introduced next.

Finally, it is worth emphasizing that even when the underlying template dynamics
are modeled exactly (e.g., nonlinear centroidal or SRB equations), they remain a
reduced-order approximation of the full quadruped dynamics and contact
interaction. As a consequence, systematic discrepancies can persist due to
neglected joint dynamics and terrain-dependent effects.
A residual Koopman augmentation is therefore valuable even in conjunction with any nominal model: by learning correction terms directly from measured state
transitions, the predictor can compensate for repeatable modeling error while
preserving an MPC-compatible linear structure in the lifted space. In this way,
the residual operator improves multi-step prediction fidelity relative to a
baseline reduced-order model, which in turn can enhance tracking performance and
robustness of the overall hierarchical locomotion controller.

\subsection{Residual Koopman Representation}
\label{subsec:residual_koopman_rep}

Let $x_{k}$ denote the measured state from the true system (or a
high-fidelity simulator). The one-step modeling error relative to the nominal predictor is defined as
\begin{equation}
\label{eq:residual_error}
e_k := x_{k} - x_{k}^{\mathrm{nom}}.
\end{equation}
Rather than replacing the nominal model, our goal is to learn a compact
predictor for $e_k$ that corrects systematic mismatch while retaining the nominal structure and interpretability. In practice, we restrict the residual to centroidal twist: $e_k := [\Delta v_k^\top \; \Delta \omega_k^\top]^\top$.

Motivated by Koopman operator theory, we assume the residual admits an
approximately linear evolution in a lifted observable space. Let
\[
z_k := \psi(e_k) \in \mathbb{R}^{q}
\]
denote a vector of residual observables (lifting functions). We model the
residual through the linear system

\begin{equation}
\label{eq:residual_lifted}
z_{k+1} = A^{\mathrm{res}} z_k + B^{\mathrm{res}} u_k,
\qquad
\hat e_k = C^{\mathrm{res}} z_k,
\end{equation}
where $A^{\mathrm{res}} \in \mathbb{R}^{q\times q}$, $B^{\mathrm{res}} \in \mathbb{R}^{q\times m}$, and
$C^{\mathrm{res}} \in \mathbb{R}^{n\times q}$ are learned from data. Combining
\eqref{eq:nominal_ltv}--\eqref{eq:residual_lifted} yields the corrected predictor
\begin{align}
\label{eq:corrected_predictor}
\hat x_{k+1}
&=
x_{k+1}^{\mathrm{nom}} + \hat e_{k+1}
=
A^{\mathrm{nom}} x_k
+
B^{\mathrm{nom}} u_k
+
C^{\mathrm{res}} z_{k+1} \nonumber \\
&=A^{\mathrm{nom}} x_k
+
B^{\mathrm{nom}} u_k
+
C^{\mathrm{res}} \left( A^{\mathrm{res}} z_k +  B^{\mathrm{res}} u_k \right).
\end{align}
This decomposition separates the modeling task into a nominal component that
encodes the SRB template structure and a data-driven residual component that
captures unmodeled effects via a linear operator in lifted coordinates.

Given a dataset of transitions $\{(x_k,u_k)\}_{k=1}^{N}$, we compute
$x_{k}^{\mathrm{nom}}$ from \eqref{eq:nominal_ltv} and form residual targets
$e_k$ using \eqref{eq:residual_error}. The lifted residual snapshots are
constructed as $z_k=\psi(e_k)$. Then, we use EDMD to identify the predictors by minimizing the lifted prediction error 
\begin{equation}
\label{eq:residual_ls}
\min_{A^{\mathrm{res}},B^{\mathrm{res}}}\;
\sum_{k=1}^{N-1}
\left\|
z_{k+1} - A^{\mathrm{res}} z_k - B^{\mathrm{res}} u_k
\right\|_2^2,
\end{equation}
followed by a linear regression for the output map
\begin{equation}
\label{eq:residual_output_ls}
\min_{C^{\mathrm{res}}}\;
\sum_{k=1}^{N-1}
\left\|
e_z - C^{\mathrm{res}} z_k
\right\|_2^2.
\end{equation}

We solve \eqref{eq:residual_ls}--\eqref{eq:residual_output_ls} via ridge-regularized least squares as in Section~\ref{sec:EDMD}, using regularization weight $\lambda>0$ for numerical robustness. In the following, we apply this residual Koopman framework to quadruped locomotion and highlight practical considerations that arise when the predictor is embedded in a hierarchical control pipeline.

Next, we present Theorem \ref{thm:rk_error_propagation}, which provides formal guarantees on the error bounds obtained using the proposed residual Koopman model.
We first make the following assumptions:
\begin{assumption} \label{assumption1}
\begin{enumerate}
\item[]
\item[(A1)] The dynamics \(F\) are Lipschitz in the state uniformly in \(u\), i.e., there exists \(L>0\) such that
\begin{equation}
\|F(x,u)-F(y,u)\| \le L \|x-y\|,
\quad \forall x,y,\; \forall u . \nonumber
\end{equation}

\item[(A2)] The residual Koopman approximation error is uniformly bounded, i.e., there exists \(\varepsilon>0\) such that
\begin{equation}
\left\|
e_{k+1}
-
C^{\mathrm{res}}
\left(
A^{\mathrm{res}} z_k + B^{\mathrm{res}} u_k
\right)
\right\|
\le \varepsilon
\qquad \forall k. \nonumber
\end{equation}
\end{enumerate}
\end{assumption}

\begin{theorem}[]
\label{thm:rk_error_propagation}
Let the true discrete-time system evolve as 
\begin{equation}
    x_{k+1} = F(x_k,u_k). \nonumber
\end{equation}
Let the nominal predictor be given by \eqref{eq:nominal_ltv}
with the residual defined in \eqref{eq:residual_error} and the residual Koopman predictor defined in \eqref{eq:residual_lifted}. The combined state predictor is
\begin{equation}
\hat x_{k+1}
=
A^{\mathrm{nom}} x_k + B^{\mathrm{nom}} u_k
+
C^{\mathrm{res}}
\left(
A^{\mathrm{res}} z_k + B^{\mathrm{res}} u_k
\right), \nonumber
\end{equation}
with the state prediction error
\begin{equation}
\tilde x_k := x_k - \hat x_k. \nonumber
\end{equation}
Then, using Assumption \ref{assumption1}, there exists a constant \(\gamma>0\), depending only on
\(A^{\mathrm{res}},B^{\mathrm{res}},C^{\mathrm{res}}\) and \(\psi\), such that

\begin{equation}
\|\tilde x_{k+1}\|
\le
2L \|\tilde x_k\|
+
\varepsilon.
\end{equation}

Consequently, over any horizon \(N\ge 1\),

\begin{equation}
\|\tilde x_{k+N}\|
\le
(2L)^N \|\tilde x_k\|
+
\sum_{i=0}^{N-1}
L^{N-1-i}
\left(
\varepsilon
\right).
\end{equation}
\end{theorem}

\begin{proof}
By the definition of the true and predicted dynamics,
\begin{align}
\tilde x_{k+1}
&=
x_{k+1} - \hat x_{k+1} \nonumber \\
&=
F(x_k,u_k)
-
\Big(
A^{\mathrm{nom}} x_k + B^{\mathrm{nom}} u_k \nonumber \\
&\quad+
C^{\mathrm{res}}
\left(
A^{\mathrm{res}} z_k + B^{\mathrm{res}} u_k
\right)
\Big). \nonumber
\end{align}
Add and subtract \(F(\hat x_k,u_k)\) to obtain
\begin{align}
\tilde x_{k+1}
&=
\Big(F(x_k,u_k)-F(\hat x_k,u_k)\Big) \nonumber \\ 
&\quad+
\Big(
F(\hat x_k,u_k)
-
A^{\mathrm{nom}} x_k
-
B^{\mathrm{nom}} u_k \nonumber \\
&\quad-
C^{\mathrm{res}}(A^{\mathrm{res}} z_k + B^{\mathrm{res}} u_k)
\Big).
\end{align}
Taking norms and applying the triangle inequality gives
\begin{align}
\|\tilde x_{k+1}\|
&\le
\|F(x_k,u_k)-F(\hat x_k,u_k)\|
 \nonumber \\
&\quad+
\Big\|
F(\hat x_k,u_k)
-
A^{\mathrm{nom}} x_k
-
B^{\mathrm{nom}} u_k \nonumber \\
&\quad-
C^{\mathrm{res}}(A^{\mathrm{res}} z_k + B^{\mathrm{res}} u_k)
\Big\|.
\end{align}
Using Assumption (A1),
\begin{equation}
\|F(x_k,u_k)-F(\hat x_k,u_k)\|
\le
L \|x_k-\hat x_k\|
=
L \|\tilde x_k\|.
\end{equation}
For the second term, observe that
\begin{equation}
e_{k+1}
=
x_{k+1} - x_{k+1}^{\mathrm{nom}}
=
F(x_k,u_k) - \left(A^{\mathrm{nom}} x_k + B^{\mathrm{nom}} u_k\right). \nonumber
\end{equation}
Then, add and subtract $F(\hat{x}_k,u_k)$ and using Assumption (A2) we obtain,
\begin{equation}
\|\tilde x_{k+1}\|
\le
L\|\tilde x_k\|
+
\varepsilon
+
\left\|
F(\hat x_k,u_k)-F(x_k,u_k)
\right\|.
\end{equation}
Applying Assumption (A1) once more yields
\begin{equation}
\left\|
F(\hat x_k,u_k)-F(x_k,u_k)
\right\|
\le
L \|\tilde x_k\|.
\end{equation}
Thus,
\begin{equation}
\|\tilde x_{k+1}\|
\le
2L\|\tilde x_k\| + \varepsilon
\end{equation}

The multi-step estimate follows by recursively applying the one-step bound:
\begin{align}
\|\tilde x_{k+N}\|
&\le
(2L)^N \|\tilde x_k\|
+
\sum_{i=0}^{N-1}
L^{N-1-i}
\left(
\varepsilon
\right).
\end{align}

This completes the proof.
\end{proof}

\begin{remark}
\label{rem:rk_significance}
Theorem~\ref{thm:rk_error_propagation} provides a formal justification for the residual Koopman approach. Unlike full-state Koopman models, the prediction error depends on the residual \(e_k\), which is typically small since the nominal model captures the dominant dynamics. As a result, the learned correction remains well-behaved, data-efficient, and stable over multi-step prediction. This explains the improved accuracy and robustness observed in practice and supports interpreting the residual Koopman model as a structured representation of model uncertainty within MPC.
\end{remark}


\section{Practical Considerations for Quadruped Locomotion}
\label{subsec:practical_considerations}

While the residual Koopman framework above is general, applying Koopman-based predictors to quadruped locomotion within hierarchical control architectures introduces several practical considerations. Locomotion dynamics depend on contact mode and planned foothold sequences, which are commonly treated as scheduled parameters (exogenous signals) rather than states. These features make dictionary design and regression more delicate, since the lifted predictor must remain consistent under (i) orientation representations used in observables, (ii) parameter scheduling induced by planned footholds, and (iii) the chosen control parametrization (e.g., wrenches versus per-foot GRFs), motivating the following discussion.

\subsection{Orientation representations in observables}
Koopman-based frameworks for SRB dynamics often build observables from Euler angles due to their minimal dimensionality and ease of implementation. However, Euler-angle parameterizations are coordinate dependent: they exhibit singularities and can accumulate representation error over multi-step prediction, leading to rotational inconsistencies when the learned predictor is rolled out over an MPC horizon.

A more geometrically faithful alternative is to construct observables directly on ${SO}(3)$, for example by including the rotation matrix $R_k \in {SO}(3)$ (or ${SE}(3)$-consistent embeddings) in the lifting, as explored in Koopman-based modeling and control on $\mathrm{SE}(3)$ and ${SO}(3)$ \cite{zinage2021koopman,chen2022koopman}. While such choices avoid Euler-angle artifacts, they introduce important practical issues for data-driven identification. First, the linear predictor obtained from regression does not, in general, preserve the nonlinear manifold constraints of ${SO}(3)$. In particular, a multi-step rollout may yield $\hat R_k$ that
drifts away from ${SO}(3)$ (loss of orthogonality and/or $\det(\hat R_k)\neq 1$), even if the one-step fit is accurate. Enforcing $\hat R_k \in {SO}(3)$ during prediction typically requires additional mechanisms such as projection (e.g., nearest orthogonal matrix via SVD), constrained regression, or normalization, which can introduce extra computation and may accumulate error over long horizons.

Second, using $R$ in the lifting increases the dimension of the observable vector: the rotation matrix is typically included through its vectorized form (e.g., $\mathrm{vec}(R)\in\mathbb{R}^9$), so each observable derived from the rotation matrix contributes nine scalar entries to the lifted state. As additional nonlinear features or cross terms are added, the lifted dimension can grow rapidly, increasing data requirements for reliable regression and computational burden in real-time MPC.

\subsection{Footstep locations as scheduled parameters}
In SRB template models, the nominal contact locations enter as scheduled parameters (through moment arms) rather than as states. In particular, the linearized SRB dynamics depend on the planned yaw and nominal footstep locations $A_k(\psi_k)$ and $B_k(\psi_k,\bar r_k)$, where $\bar r_k$ is provided by a footstep planner. In common hierarchical frameworks, the MPC
receives these future foothold references from an upstream planner and treats them as fixed over the horizon. As a result, the predicted state evolution is \emph{conditionally defined} on the planned foothold sequence: even for identical $(x_k,u_k)$, changing $\bar r_k$ changes the nominal state transition through $B_k(\psi_k,\bar r_k)$. This creates a structural source of mismatch for Koopman-based predictors: if the lifting $\psi(\cdot)$ or the learned linear dynamics are conditioned on planned footholds, then multi-step prediction accuracy hinges on the consistency between (i) the planned footholds used to generate the predictor and (ii) the realized contacts on the robot. Deviations in realized footholds due to slip, compliance, foothold erosion, or unmodeled terrain geometry effectively induce a parameter mismatch.

Concretely, a predictor constructed using the planned foothold sequence $\bar r_k$ will generally incur residuals when the realized contact moment arms deviate from plan. Let $\bar{r}_k^{\mathrm{\;real}}$ denote the realized moment arms and define the foothold deviation $\varepsilon_k :=\bar{r}_k^{\mathrm{\;real}}-\bar r_k$. Then, even for the same state and force input $u_k$, the scheduled input mapping satisfies $B_k(\psi_k,\bar{r}_k^{\mathrm{\;real}})u_k \neq B_k(\psi_k,\bar r_k)u_k$ whenever $\varepsilon_k \neq 0$. Over a prediction horizon, such scheduling mismatch can accumulate, particularly in off-road settings where foothold uncertainty is high and contact conditions vary rapidly. This issue is especially relevant for Koopman-based quadruped MPC pipelines that learn linear predictors under nominal gait and foothold patterns \cite{yang2025koopman}, and motivates modeling strategies that either (i) explicitly account for foothold uncertainty in the predictor, or (ii) adapt the
lifted model online/continually as contact conditions evolve \cite{li2025continual}.

\subsection{Control parametrization}
A physics-informed choice of observables for SRB centroidal dynamics naturally suggests representing the input as a net body wrench (total force and torque).
This choice is common in Koopman-based modeling and control frameworks for rigid-body motion \cite{narayanan2023se,rajkumar2025real,zinage2022koopman}. For quadrupeds, however, the net wrench $\left(F, \; \tau \right)$ is not a primitive control input: it is realized through distributed ground reaction forces across multiple contacts, and the mapping from per-foot GRFs to the net wrench depends on the contact mode and foothold geometry (i.e., $ \sum f_i, \textrm{and} \; \sum r_i \times f_i$ as given in \eqref{eq:force_map}). Moreover, when multiple feet are in contact, this mapping is generally non-invertible, implying that a desired wrench corresponds to a family of force distributions whose feasibility is governed by unilateral contact and friction constraints at each foot. 

Consequently, adopting a Koopman-based predictor based on wrench inputs can misalign the learned model with the standard legged MPC formulations that optimize per-foot GRFs subject to contact constraints. Ensuring consistency typically requires either (i) explicitly introducing the per-foot forces $(f_i)$ as decision variables with additional constraints linking them to the net wrench, or (ii) an auxiliary feasibility/projection step that maps an optimized wrench to admissible GRFs, which can add computation and increase model-mismatch when the desired wrench is not realizable.
\subsection{Implications for residual Koopman representations}
\label{subsec:implications_residual}

The above considerations suggest that, for contact-rich quadruped locomotion, Koopman-based predictors are most effective when used as \emph{structured residual models} that preserve the nominal SRB dynamics and their planner-conditioned scheduling variables, rather than attempting to learn a lifted model of the full nonlinear closed-loop system. In particular, we leverage a nominal SRB-based predictor that captures the dominant centroidal trends for \emph{flat-terrain walking} under nominal contact assumptions, and learn a Koopman residual model that corrects systematic mismatch arising in more challenging operating conditions.

Concretely, we treat the nominal dynamics as a physics-based backbone and fit a data-driven residual map from measured transitions, so that the learned residual captures unmodeled effects due to changing ground friction, terrain geometry, and contact variability while remaining compatible with the convex MPC structure. This residual viewpoint retains interpretability (corrections are applied only where the nominal model is insufficient), reduces the learning burden relative to full-state lifting, and enables the same MPC formulation to generalize beyond the nominal regime.

\begin{remark}
\label{rem:rk_disturbance}
The residual Koopman model can also be interpreted as learning a linear disturbance generator in lifted coordinates. In particular, the term 
\(
C^{\mathrm{res}}(A^{\mathrm{res}} z_k + B^{\mathrm{res}} u_k)
\)
acts as a structured, data-driven correction to the nominal dynamics, analogous to a disturbance model that evolves linearly in the lifted space. This perspective clarifies the role of the residual as capturing systematic model mismatch (e.g., contact variability and terrain effects) while preserving a control-oriented linear structure, and naturally connects the proposed framework to disturbance-based and robust MPC formulations.
\end{remark}

In the next section, we describe how we generate diverse training data by executing the nominal walking controller across randomized friction coefficients and a collection of terrain maps, and then identify a Koopman predictor for the one-step residual. We subsequently evaluate the learned residual model on held-out trajectories to assess predictive structure and accuracy, before incorporating it into the MPC framework and presenting closed-loop locomotion results.


\section{Residual Koopman Model Identification}
\label{sec:residual_id}
\begin{figure*}
    \centering
    \includegraphics[width=1\linewidth]{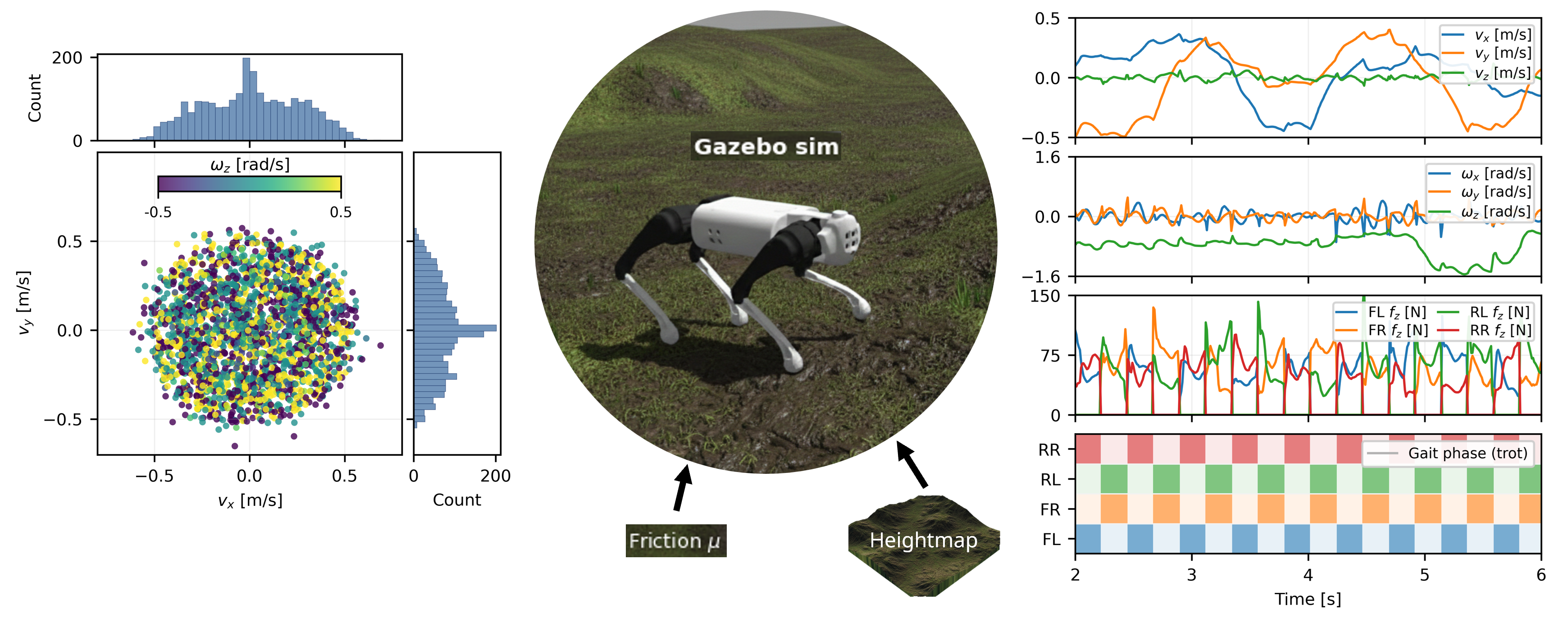}
    \caption{\textbf{Dataset generation for Koopman residual learning.} Left: body-velocity samples $(v_x,v_y)$ with marginal histograms, colored by yaw rate $\omega_z$, illustrating excitation coverage. Middle: Gazebo simulation setup with randomized off-road terrain with per-episode friction $\mu$. Right: representative episode time histories of base linear velocity $v$, angular velocity $\omega$, vertical contact forces $f_z$, and trot gait phase.}
    \label{fig:training_data}
\end{figure*}

In this section, we identify and evaluate the proposed residual Koopman predictor for quadruped locomotion on a Unitree Go1 robot in simulation. We first describe the data collection procedure used to excite a broad range of contact conditions (via randomized friction and terrain variations). We assess predictive performance on held-out trajectories and use the learned predictor in the subsequent MPC formulation to demonstrate closed-loop locomotion over diverse operating conditions.

\subsection{Data collection and preprocessing}
\label{subsec:data_collection}

Training data were generated in simulation using the Gazebo physics engine and a nominal locomotion controller. We collected $10$ independent episodes (approximately $2$ minutes each) at $100$~Hz ($\Delta t=10$~ms). After discarding a short initialization transient and time-aligning states and inputs, the resulting dataset contains $N=121{,}753$ synchronized samples for identification and evaluation. To promote diversity and facilitate sim-to-real transfer, each episode was executed under randomized terrain conditions. The ground friction coefficient was sampled independently per episode from a uniform distribution on $[0.5,\,1.0]$. Terrain geometry was also randomized by alternating between flat ground and mild heightmap-based rough terrain. The rough terrain was generated by low-pass filtering zero-mean Gaussian noise to create spatially correlated height variations, exporting the result as a grayscale PNG heightmap, and scaling it in Gazebo to produce small-amplitude surface roughness. 

During each episode, the nominal controller tracked randomized planar velocity commands. Setpoints were sampled uniformly within $v_x^{\mathrm{cmd}},v_y^{\mathrm{cmd}}\in[-0.7,\,0.7]$~m/s and $\omega_z^{\mathrm{cmd}}\in[-0.5,\,0.5]$~rad/s at a fixed update period and filtered with a first-order low-pass filter to ensure smooth reference profiles. Fig.~\ref{fig:training_data} summarizes the resulting coverage in $(v_x,v_y)$ with $\omega_z$ encoded by color and shows a representative rollout segment with measured velocities, reference commands, GRF inputs, and the trot gait schedule.

At each time step, we log the base position $p_k$ and attitude $R_k\in SO(3)$ together with the base linear and angular velocities $(v_k,\omega_k)$ using a linear Kalman‑filter state estimator (we use the same state estimator in both simulation and hardware). The attitude $R_k$ is obtained by converting the estimator‑provided quaternion to a rotation matrix. All logged velocities are expressed in the world frame. As control inputs, we log the per‑foot ground reaction forces (GRFs) predicted by the nominal controller, stacked as $u_k=[f_{1,k}^\top,\ldots,f_{4,k}^\top]^\top\in\mathbb{R}^{12}$. We also log the planner‑provided moment arms $r_{i,k}$. Finally, we preprocess the dataset by applying a per‑feature standardization (row‑wise $z$‑score scaling) to mitigate unit and magnitude disparities between lifted observables and GRF inputs. We report simple numerical diagnostics to characterize dataset diversity and excitation. Let
$Y:=\begin{bmatrix}y_0&\cdots&y_{N-1}\end{bmatrix}$ denote a stacked signal matrix (planar twist with $y_k=[v_{x,k},\,v_{y,k},\,\omega_{z,k}]^\top$), and evaluate its singular values, numerical rank, and condition number $\kappa(Y)=\sigma_{\max}(Y)/\sigma_{\min}(Y)$. For the planar twist, we obtain $\mathrm{rank}(Y)=3$ with singular values $\{147.58,\ 63.32,\ 59.77\}$ and $\kappa(Y)\approx 2.47$, indicating well‑conditioned coverage in $(v_x,v_y,\omega_z)$.

For all EDMD-based models, we use the stacked per-foot GRFs as the control input $u_k=[f_{1,k}^\top,\ldots,f_{4,k}^\top]^\top$, while the foothold moment arms $r_{i,k}$ enter the nominal predictor as exogenous, time-varying (planner-provided) parameters. All regressions are solved via ridge-regularized least squares (Section~\ref{sec:EDMD}) with $\lambda = 10^{-6}$, and we apply per-feature standardization to the EDMDc regressor $\Omega$.

\subsection{Model Identification}
\label{subsec:residual_edmd_monomials}

\begin{table}[t]
\centering
\caption{Open-loop prediction RMSE (mean $\pm$ std).}
\label{tab:rmse_all_models}
\footnotesize
\setlength{\tabcolsep}{2.2pt}
\renewcommand{\arraystretch}{0.90}
\begin{tabular}{lcccc}
\toprule
State & SRB & EDMD (mono) & EDMD-SE(3) & Res-Koopman \\
\midrule
\multicolumn{5}{l}{\textit{Linear velocity (m/s)}} \\
\midrule
$v_x$ & $0.007\!\pm\!0.003$ & $4.05\!\pm\!1.92$ & $0.003\!\pm\!0.001$ & $0.003\!\pm\!0.001$ \\
$v_y$ & $0.008\!\pm\!0.002$ & $3.19\!\pm\!1.63$ & $0.003\!\pm\!0.001$ & $0.003\!\pm\!0.001$ \\
$v_z$ & $0.016\!\pm\!0.003$ & $1.22\!\pm\!0.55$ & $0.009\!\pm\!0.002$ & $0.007\!\pm\!0.002$ \\
\midrule
\multicolumn{5}{l}{\textit{Angular velocity (rad/s)}} \\
\midrule
$\omega_x$ & $0.091\!\pm\!0.018$ & $3.16\!\pm\!1.91$ & $0.070\!\pm\!0.022$ & $0.064\!\pm\!0.017$ \\
$\omega_y$ & $0.118\!\pm\!0.042$ & $2.11\!\pm\!1.06$ & $0.094\!\pm\!0.044$ & $0.089\!\pm\!0.028$ \\
$\omega_z$ & $0.035\!\pm\!0.005$ & $1.29\!\pm\!0.50$ & $0.016\!\pm\!0.005$ & $0.023\!\pm\!0.005$ \\
\bottomrule
\end{tabular}
\end{table}

\begin{figure*}
  \centering
  \includegraphics[width=\linewidth]{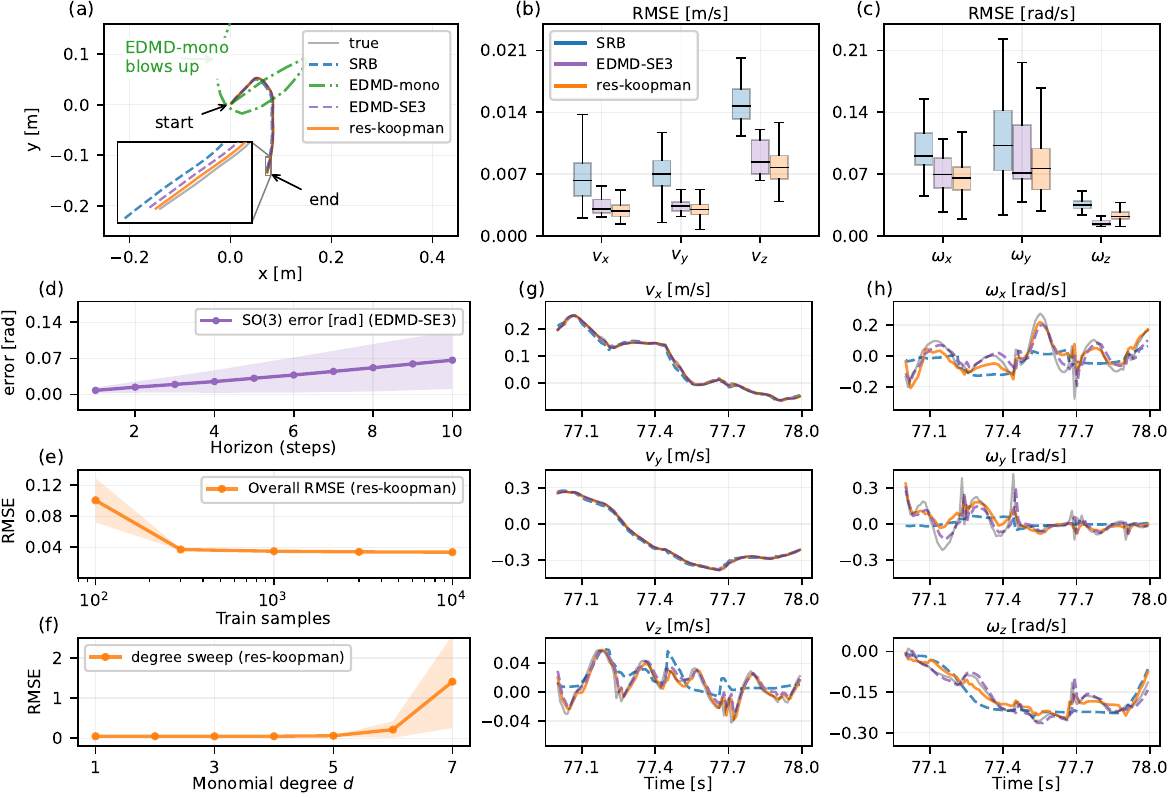}
    \caption{\textbf{Prediction performance:} Comparison of the proposed residual Koopman model (orange) against a nonlinear SRB (blue dashed), EDMD with monomials (green dashed) and EDMD with SE(3) basis (purple dashed against a test trajectory (black). (a) prediction of $x-y$ trajectory over 100 time steps. (b--c) Per-channel RMSE box plots for linear velocities $(v_x,v_y,v_z)$ and angular velocities $(\omega_x,\omega_y,\omega_z)$, 
    (d) Mean $SO(3)$ geodesic attitude error of EDMD-$SE(3)$ over multi-step rollout horizon, 
    (e) Overall residual Koopman RMSE versus the number of training samples used and 
    (f) Overall RMSE error against monomial degree sweep for the residual Koopman model.}

  \label{fig:residual_edmd_fit}
\end{figure*}

In the proposed \textit{residual Koopman} framework, we first construct the residual targets $e_k$ using \eqref{eq:residual_error} from measured transitions and the nominal SRB predictor \eqref{eq:nominal_ltv}. The residual is defined over the centroidal twist states as $e_k := [\,\Delta v_k^\top\;\; \Delta\omega_k^\top\,]^\top \in \mathbb{R}^{6}$, where $\Delta v_k := v_{k+1}-v_{k+1}^{\mathrm{nom}}$ and $\Delta\omega_k := \omega_{k+1}-\omega_{k+1}^{\mathrm{nom}}$. Then, we form lifted snapshots $z_k=\psi(e_k)$ using a monomial dictionary up to degree~$d$,
\begin{equation}
\psi(e_k)
\;:=\;
\big[\,1,\; e_k^\top,\; (e_k^{2})^\top,\; \ldots,\; (e_k^{d})^\top\big]^\top
\in \mathbb{R}^{q},
\label{eq:residual_lift_monomial}
\end{equation}
where $e_k^{j}$ denotes the vector of all degree-$j$ monomials formed from the components of $e_k$. We apply EDMDc to estimate $(A^{\mathrm{res}},B^{\mathrm{res}})$ by minimizing the lifted one-step prediction error in \eqref{eq:residual_ls}, followed by the output regression \eqref{eq:residual_output_ls} to identify $C^{\mathrm{res}}$. We select the residual observable lift $z_k=\psi(e_k)$ as a degree-$2$ monomial dictionary based on the polynomial-degree sweep shown in Fig.~\ref{fig:residual_edmd_fit}{(f)}. After per-feature standardization, the EDMDc regressor $\Omega$ is full row rank with $\mathrm{rank}(\Omega)=q+m$ and well-conditioned with $\kappa(\Omega)=5.03$, supporting numerically stable estimation of the lifted predictor. Although the lifted state can be high-dimensional, the learned correction is always applied in the physical space by projecting back to the $6$-dimensional residual via $e_k \approx C^{\mathrm{res}} z_k$. 

To assess data requirements and guide model selection, Fig.~\ref{fig:residual_edmd_fit}{(e)} reports sample-efficiency results obtained by fitting the residual EDMDc model on random subsets of the training data and evaluating the overall windowed one-step RMSE on a fixed test set (50 windows). Performance improves rapidly with additional samples and largely saturates around 1000 transitions, decreasing from $0.3759\pm 0.1385$ at 100 samples to $0.0403\pm 0.0046$ at 1000 samples, with diminishing returns thereafter ($0.0349\pm 0.0035$ at 3000 and $0.0333\pm 0.0032$ at 10{,}000). Fig.~\ref{fig:residual_edmd_fit}{(f)} further shows a polynomial-degree sweep for $z_k=\psi(e_k)$ using degree-$d$ monomials in the $6$-dimensional residual $e_k$: while lower degrees achieve similar test error, higher degrees substantially increase the lifting dimension $q$ and degrade generalization despite stronger regularization. Table~\ref{tab:rmse_all_models} reports test-set RMSE for one-step prediction over $100$-step rollouts, averaged across $100$ randomly sampled windows from held-out trajectories.

\subsection{Prediction performance comparison with baseline EDMD}

In Fig.~\ref{fig:residual_edmd_fit} and Table~\ref{tab:rmse_all_models}, we compare the proposed residual Koopman predictor (orange) against the nominal nonlinear SRB model (blue dashed), and two EDMD baseilens: degree-4 monomial EDMD (green dashed), and SE(3)-based EDMD \cite{narayanan2023se,zinage2021koopman} (purple dashed). All predictions are evaluated over $100$-step open-loop rollouts, sampled at random from held-out test trajectories. RMSE is reported as the average one-step prediction error computed over $100$ time steps and then averaged across $100$ randomly selected windows. For both EDMD baselines, we identify a linear predictor in a lifted space as in Section~\ref{sec:EDMD}. The monomial baseline uses
\begin{equation}
\Psi_{\mathrm{mono}}(x)
=
\big[\,1,\; x^\top,\; (x^{2})^\top,\; \ldots,\; (x^{d})^\top \big]^\top,
\label{eq:EDMD_monomial}
\end{equation}
with $x := [\,p^\top,\;\theta^\top,\;v^\top,\;\omega^\top\,]^\top$, where $x^{d}$ denotes the vector of all degree-$d$ monomials. As shown by the green trajectory in Fig.~\ref{fig:residual_edmd_fit}(a) and the RMSE statistics in Table~\ref{tab:rmse_all_models}, this monomial EDMD model exhibits large prediction errors and can become unstable over multi-step rollouts, yielding poor one-step and long-horizon performance.

We next consider an SE(3)-structured EDMD baseline that incorporates rotation-aware features. Specifically, we form $\psi_{SE(3)}(x)=\big[\,p^\top,\;v^\top,\;\mathrm{vec}(R)^\top,\;\omega^\top,\;h(x)\big]$, where
\begin{equation}
h(x)
=
\Big[\,
\mathrm{vec}(R{\omega}^\wedge)^\top,
\mathrm{vec}(R{\omega^\wedge}^{2})^\top,
\ldots,
\mathrm{vec}(R{\omega^\wedge}^{d})^\top
\Big]^\top, 
\label{eq:EDMD_SE3}
\end{equation}
with $d=4$ (degree~4). This dictionary avoids the catastrophic blow-up of monomial EDMD and achieves substantially lower RMSE (Table~\ref{tab:rmse_all_models} and Fig.~\ref{fig:residual_edmd_fit}(b,c)). However, because the lift is constructed from the estimated attitude, small orientation errors perturb the lifted coordinates and compound over multi-step rollouts, leading to horizon-dependent attitude drift. We quantify this effect using the $SO(3)$ geodesic distance \cite{barfoot2024state}
\begin{equation}
e_R \;:= \bigg\|tr\left(\mathbb{I}-R^\top \hat{R} \right) \bigg\|_2,\nonumber
\end{equation}
which increases approximately linearly with the prediction horizon in Fig.~\ref{fig:residual_edmd_fit}(d), indicating a distinct failure mode when used in a predictive control setting.

In the $xy$ plane (Fig.~\ref{fig:residual_edmd_fit}(a)), the residual Koopman predictor remains closest to the measured trajectory, while the nominal SRB model exhibits mild drift under heightmap-induced contact variability and the degree-$4$ monomial EDMD model diverges. The corresponding time-series (Fig.~\ref{fig:residual_edmd_fit}(g,h)) show that all models track $v_x$ and $v_y$ closely over this segment, but the residual Koopman predictor most accurately captures the $v_z$ fluctuations induced by vertical motion and contact modulation, and provides the tightest overall tracking of $(\omega_x,\omega_y,\omega_z)$. These trends are consistent with the aggregate RMSE box plots in Fig.~\ref{fig:residual_edmd_fit}(b,c) and Table~\ref{tab:rmse_all_models}, and motivate learning a compact residual correction on top of a physically meaningful nominal predictor rather than relying on a single lifted EDMD model of the full state.

These results motivate using the learned residual Koopman model as a lightweight correction term within a predictive controller: it preserves the interpretability and structure of the SRB template while improving short-horizon prediction fidelity under contact variability. We next incorporate this residual Koopman predictor into our MPC formulation.

\section{Convex Residual Koopman MPC for Quadruped Locomotion}
\label{sec:MPC}

We now embed the residual Koopman predictor from Section~\ref{sec:residual_koopman} into a receding-horizon MPC. Based on the linearized nominal SRB template \eqref{eq:nominal_ltv} we design an MPC formulation that uses the residual-corrected linear predictor while preserving a convex QP structure.

\subsection{Residual Koopman MPC Formulation}
\label{subsec:reskoop_mpc}

\begin{figure*}[t]
\vspace{-1mm}
\hrule\vspace{0.8mm}
{\small\textit{Residual Koopman MPC formulation.}}
\begin{subequations}
\label{eq:reskoop_mpc_ocp}
\begin{align}
\min_{\{x_{k+i},\,u_{k+i}\}} \quad
& \sum_{i=0}^{N-1}
\Big(
(x_{k+i+1}-x_{k+i+1}^\star)^\top Q (x_{k+i+1}-x_{k+i+1}^\star)
+
u_{k+i}^\top R\,u_{k+i}
\Big)
\nonumber\\[-0.5mm]
&\hspace{3.2em}+
(x_{k+N}-x_{k+N}^\star)^\top Q_f (x_{k+N}-x_{k+N}^\star)
\label{eq:reskoop_mpc_obj}\\[0.5em]
\text{s.t.}\quad
& x_{k+i+1}=A^{\mathrm{nom}}x_{k+i}+B^{\mathrm{nom}}u_{k+i}+C^{\mathrm{res}}(A^{\mathrm{res}} z_{k+i}+ B^{\mathrm{res}} u_{k+i}),
\label{eq:reskoop_mpc_xdyn}\\
& A_{\mathrm{ineq}}(\sigma_{k+i})\,u_{k+i}\le b_{\mathrm{ineq}}(\sigma_{k+i}),
\label{eq:reskoop_mpc_ineq} \\
& z_k=\psi(e_k),\quad
e_k=\big[(v_k-v_k^{\mathrm{nom}})^\top\;(\omega_k-\omega_k^{\mathrm{nom}})^\top\big]^\top,
\label{eq:reskoop_mpc_ic}
\end{align}
\end{subequations}
\vspace{0.8mm}\hrule
\vspace{-2mm}
\end{figure*}

Using the residual Koopman representation \eqref{eq:residual_lifted}, the MPC propagates the lifted residual state $z_k=\psi(e_k)$ alongside the nominal SRB prediction. The combined predictor is
\begin{align}
\hat x_{k+1}
&=
x_{k+1}^{\mathrm{nom}} + \hat e_{k+1} \nonumber \\
&=
A^{\mathrm{nom}} x_k + B^{\mathrm{nom}} u_k + C^{\mathrm{res}} z_{k+1}, 
\label{eq:mpc_corrected_clean}\\
z_{k+1}
&=
A^{\mathrm{res}} z_k + B^{\mathrm{res}} u_k,
\qquad
\hat e_k = C^{\mathrm{res}} z_k.
\label{eq:mpc_lifted_clean}
\end{align}
Here, $x_k\in\mathbb{R}^{13}$ is the SRB state
$x_k=[\,p_k^\top\;\theta_k^\top\;v_k^\top\;\omega_k^\top\;1\,]^\top$
defined in \eqref{eqn:original_states}, and $u_k\in\mathbb{R}^{6}$ is the net centroidal wrench applied to the base,
$u_k=[\,F_k^\top\;\tau_k^\top\,]^\top$.
In our implementation, the residual is defined only over the velocity states,
\begin{align}
e_k
:=
\begin{bmatrix}
v_k - v_k^{\mathrm{nom}}\\[1mm]
\omega_k - \omega_k^{\mathrm{nom}}
\end{bmatrix}
\in\mathbb{R}^6,
\label{eq:velocity_residual_def}
\end{align}
and $z_k\in\mathbb{R}^{q}$ denotes the lifted residual coordinates produced by the dictionary $\psi(\cdot)$ acting on \eqref{eq:velocity_residual_def}. Accordingly, $C^{\mathrm{res}} z_k$ in \eqref{eq:mpc_corrected_clean} is interpreted as an additive correction to the predicted velocities within the full SRB state, while the nominal model propagates $(p,\theta)$ through \eqref{eq:nominal_ltv}. The nominal matrices $A_k^{\mathrm{nom}}:=A(\psi_k)$ and $B_k^{\mathrm{nom}}:=B(\psi_k,\bar r_k)$ are as in \eqref{eq:nominal_ltv}, and $(A^{\mathrm{res}},B^{\mathrm{res}},C^{\mathrm{res}})$ are identified via \eqref{eq:residual_ls}--\eqref{eq:residual_output_ls}.


To track a reference trajectory $\{x_{k+i}^\star\}_{i=0}^{N}$, we solve a receding-horizon optimal control problem using the residual-corrected predictor in \eqref{eq:mpc_corrected_clean}--\eqref{eq:mpc_lifted_clean}. At each time step $k$, we compute the nominal prediction $x_k^{\mathrm{nom}}$ from \eqref{eq:nominal_ltv} and form the residual state only in the velocity channels,
\[
e_k :=
\big[(v_k-v_k^{\mathrm{nom}})^\top\;(\omega_k-\omega_k^{\mathrm{nom}})^\top\big]^\top
\in \mathbb{R}^{6},
\]
which is lifted as $z_k=\psi(e_k)$ as described in \eqref{eq:reskoop_mpc_ic}. The decision variable is the sequence of stance-foot contact forces
\[
u_{k+i}:=\big[\,f_{1,k+i}^\top\;\cdots\;f_{n_c,k+i}^\top\,\big]^\top \in \mathbb{R}^{3n_c},
\]
where $n_c$ denotes the number of stance feet at stage $k+i$. The resulting Residual Koopman MPC (RK-MPC) formulation is provided in \eqref{eq:reskoop_mpc_obj}- \eqref{eq:reskoop_mpc_ic}.

Here, $Q\succeq 0$ and $R\succ 0$ weight state tracking and force effort, and $Q_f\succeq 0$ is an optional terminal weight. The mode $\sigma_{k+i}$ denotes the stance-foot set at stage $k+i$ (gait schedule), which determines which contact forces are active and which inequality constraints are enforced. For each stance foot $j\in\sigma_{k+i}$ with force $f_{j,k+i}=[f_{j,x}\;f_{j,y}\;f_{j,z}]^\top$, we enforce a linear friction pyramid approximation \cite{trinkle1997dynamic}
\[
|f_{j,x}| \le \mu f_{j,z},
\qquad
|f_{j,y}| \le \mu f_{j,z},
\]
together with vertical force bounds
\[
f_{\min} \le f_{j,z} \le f_{\max}.
\]
These constraints can be written compactly as
\begin{align}
A_{\mu}\, f_{j,k+i} \le b_{\mu},
\label{eq:friction_pyr_foot}
\end{align}
where
\begin{align}
A_{\mu} :=
\begin{bmatrix}
-1 & 0  & -\mu \\
 1 & 0  & -\mu \\
 0 & -1 & -\mu \\
 0 & 1  & -\mu \\
 0 & 0  & -1 \\
 0 & 0  &  1 \\
\end{bmatrix},
\qquad
b_{\mu} :=
\begin{bmatrix}
0\\0\\0\\0\\-f_{\min}\\ f_{\max}
\end{bmatrix}.
\label{eq:friction_pyr_mats}
\end{align}
Stacking \eqref{eq:friction_pyr_foot} over all stance feet yields the stage-wise inequality in \eqref{eq:reskoop_mpc_ineq}, i.e., $A_{\mathrm{ineq}}(\sigma_{k+i})u_{k+i}\le b_{\mathrm{ineq}}(\sigma_{k+i})$. At each control step, we apply the first optimal contact-force command $u_k^\star$ in a receding-horizon fashion.

Because the cost in \eqref{eq:reskoop_mpc_ocp} is quadratic and the constraints
\eqref{eq:reskoop_mpc_xdyn} and \eqref{eq:reskoop_mpc_ineq} are linear in the decision variables, \eqref{eq:reskoop_mpc_ocp} is a convex quadratic program. In particular, stacking the horizon inputs into
\[
U := [\,u_k^\top\;u_{k+1}^\top\;\cdots\;u_{k+N-1}^\top\,]^\top
\]
and eliminating the state variables via the recursion yields the standard form
\begin{align}
\min_{U}\quad & \tfrac{1}{2}U^\top H U + G^\top U, \nonumber\\
\text{s.t.}\quad & A_{\mathrm{ineq}}^{N} U \le b_{\mathrm{ineq}}^{N} \nonumber
\end{align}
for appropriate matrices $(H,G,A_{\mathrm{ineq}}^{N},b_{\mathrm{ineq}}^{N})$ determined by the nominal SRB matrices $(A^{\mathrm{nom}},B^{\mathrm{nom}})$, the residual Koopman operators $(A^{\mathrm{res}},B^{\mathrm{res}},C^{\mathrm{res}})$, the horizon weights $(Q,R,Q_f)$, and the contact schedule $\sigma$. This QP can be solved efficiently with standard QP solvers. Note that the nominal state $x_k^{\text{nom}}$ is propagated online alongside the measured state. 


\subsection{Tracking performance comparison with baseline EDMD}
\label{subsec:sim_tracking_compare}

\begin{figure*}
  \centering
  \includegraphics[width=\linewidth]{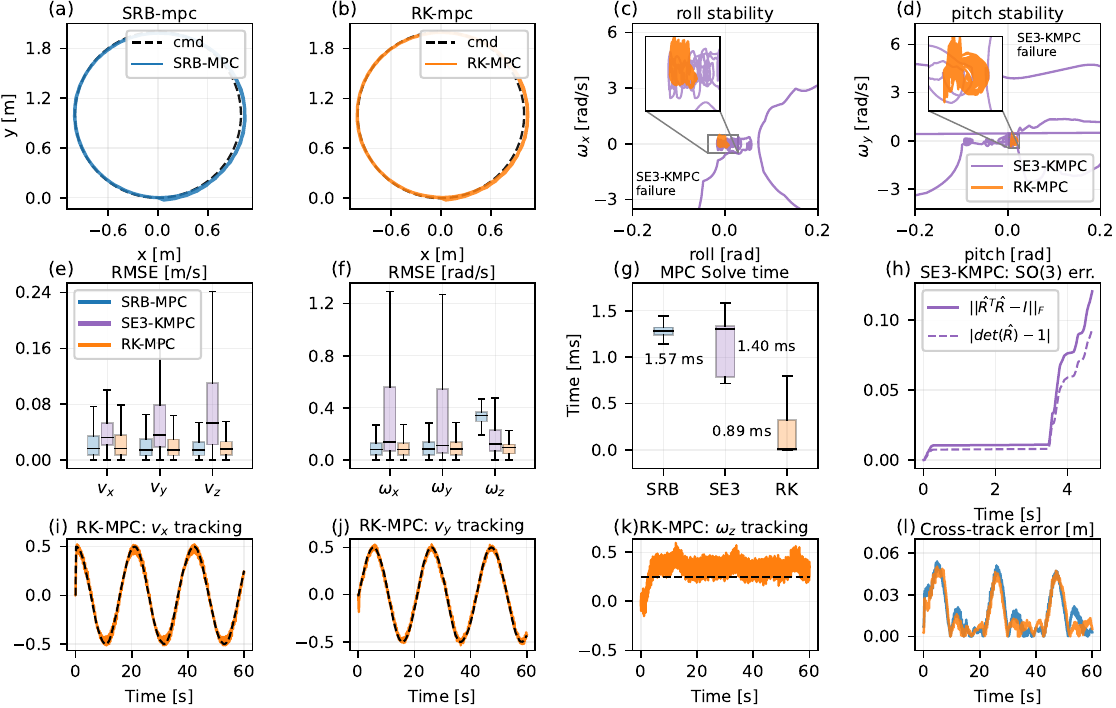}
  \caption{\textbf{MPC tracking performance (sim):} Comparison of the proposed RK-MPC framework (orange) against a nonlinear SRB-MPC (blue) and SE3-KMPC (purple) on a circular tracking task over 10 laps.
(a)--(b) Executed base $x$--$y$ trajectory of the SRB-MPC and the RK-MPC respectively.
(c)--(d) Phase portraits for attitude stability where SE3-KMPC exhibits a failure mode with excursions away from the near-origin region.
(e)--(f) Per-axis tracking RMSE distributions for linear velocities and angular rates across 10 laps of the circle.
(g) MPC solve-time distributions (ms) for each controller.
(h) Closed-loop ${SO}(3)$ drift of the SE3-KMPC internal rotation state.
(i)--(k) Tracking performance of RK-MPC for $(v_x,v_y,\omega_z)$, showing close agreement with the commanded signals.
(l) Cross-track error for RKMPC and SRB-MPC over time.}
  \label{fig:sim_mpc_somparison}
\end{figure*}

We evaluate the proposed Residual Koopman MPC (RK-MPC) from
\eqref{eq:reskoop_mpc_ocp} in Gazebo simulation on a velocity-tracking task that
executes a circular reference for $10$ laps. A smooth global planner takes the
body-frame velocity commands from user input and converts them into a world-frame CoM (base)
state reference, yielding the reference sequence $\{x_{k+i}^\star\}_{i=0}^{N}$
used in the tracking objective \eqref{eq:reskoop_mpc_obj}. All controllers are
evaluated under identical simulation conditions (same gait schedule, contact
parameters, and command profile), and use the same sampling time $\Delta t$ and
horizon length $N$. For RK-MPC, the nominal predictor inside the horizon is the
linearized SRB template \eqref{eq:nominal_ltv} discretized at $\Delta t=0.01$ s, while
the learned residual Koopman dynamics \eqref{eq:mpc_lifted_clean} provide an
additive correction to the velocity channels through
\eqref{eq:mpc_corrected_clean}. 

The RK-MPC has a state dimension $n_x=13$ and the control input is the stacked ground reaction forces of four feet, $u=[f_1^\top\,f_2^\top\,f_3^\top\,f_4^\top]^\top\in\mathbb{R}^{12}$. We use a prediction horizon of $N=8$ with the MPC update rate set to $100$~Hz. Contact constraints enforce a linear friction-pyramid model. The friction
coefficient $\mu$ is randomized by sampling uniformly from
$[0.5,\,1.0]$ and is held constant within the eperiment. To evaluate robustness to
mild terrain variation, episodes alternate between flat ground and randomly
generated heightmap terrain, where the heightmap is obtained by low-pass
filtering zero-mean Gaussian noise to produce spatially correlated, small-amplitude
surface roughness in Gazebo. The normal-force limit is upper bounded by $f_{z}\le f_{z,\max}=180$~N for stance feet. The stage cost \eqref{eq:reskoop_mpc_obj} uses diagonal weights 
\begin{align*}
Q&=\mathrm{diag}(1,1,200,30,30,1,20,20,20,1,1,1,0),\\
R&=10^{-5}\mathrm{diag}(1,1,0.1,\,1,1,0.1,\,1,1,0.1,\,1,1,0.1).
\end{align*}
with $Q_f=Q$. The resulting condensed quadratic program is solved using QuadProg++ \cite{di2007quadprog++}. All simulation experiments are run in ROS + Gazebo on Ubuntu~20.04 using a Dell OptiPlex workstation equipped with a 12th Gen Intel Core i9-12900K ($\sim$3.2~GHz) and 32~GB RAM.

\begin{table}[t]
\centering
\caption{Closed-loop tracking RMSE and MPC solve-time statistics on a 10-lap circular task in simulation.}
\label{tab:rmse_tracking_mpc}
\footnotesize
\setlength{\tabcolsep}{3.0pt}
\renewcommand{\arraystretch}{0.92}
\begin{tabular}{lccc}
\toprule
Metric & SRB-MPC & SE3-KMPC & RK-MPC \\
\midrule
\multicolumn{4}{l}{\textit{Linear velocity RMSE (m/s)}} \\
\midrule
$v_x$ & $0.0229\!\pm\!0.0201$ & $0.0686\!\pm\!0.0949$ & $0.0232\!\pm\!0.0201$ \\
$v_y$ & $0.0196\!\pm\!0.0167$ & $0.0877\!\pm\!0.1237$ & $0.0195\!\pm\!0.0165$ \\
$v_z$ & $0.0172\!\pm\!0.0126$ & $0.1374\!\pm\!0.2109$ & $0.0177\!\pm\!0.0129$ \\
\midrule
\multicolumn{4}{l}{\textit{Angular-rate RMSE (rad/s)}} \\
\midrule
$\omega_x$ & $0.0910\!\pm\!0.0695$ & $1.0646\!\pm\!2.1215$ & $0.0925\!\pm\!0.0713$ \\
$\omega_y$ & $0.0967\!\pm\!0.0758$ & $0.7421\!\pm\!1.3950$ & $0.0982\!\pm\!0.0769$ \\
$\omega_z$ & $0.3368\!\pm\!0.0698$ & $0.2936\!\pm\!0.4394$ & $0.0938\!\pm\!0.0662$ \\
\midrule
\multicolumn{4}{l}{\textit{MPC solve time (ms)}} \\
\midrule
Solve time & $1.40\!\pm\!0.72$ & $1.57\!\pm\!1.27$ & $0.89\!\pm\!6.58$ \\
\bottomrule
\end{tabular}
\end{table}

We compare the proposed RK-MPC framework against two MPC pipelines that differ
only in the prediction model. The baseline SRB-MPC
solves a nonlinear MPC using the full SRB dynamics model, which leads to a
non-convex optimization problem but serves as a high-fidelity reference. In
contrast, the EDMD-SE(3) based MPC formulation (SE3-KMPC) replaces the nonlinear SRB predictor with
a linear Koopman/EDMD predictor identified using the SE(3)-structured dictionary (with degree 4)
\eqref{eq:EDMD_SE3}. This yields a convex QP with the same horizon, tracking
weights, and contact constraints as RK-MPC, but its lifted state depends
explicitly on a propagated rotation matrix $\hat R$ (through $\mathrm{vec}(\hat
R)$ and $h(x)$). As seen in the open-loop prediction results in
Fig.~\ref{fig:residual_edmd_fit}(d), small rotation-consistency errors can
accumulate over the prediction horizon and degrade closed-loop performance.

Figure~\ref{fig:sim_mpc_somparison} and Table \ref{tab:rmse_tracking_mpc} summarizes the closed-loop tracking behavior
for a circular trajectory over 10 laps. The executed base trajectories in
Fig.~\ref{fig:sim_mpc_somparison}(a)--(b) show that SRB-MPC and RK-MPC both track
the commanded circle with small drift, with RK-MPC exhibiting slightly tighter
overlap over the full rollout. This is also reflected
in the cross-track error plot in Fig.~\ref{fig:sim_mpc_somparison}(l), where both
controllers remain bounded with comparable error profiles, and RK-MPC exhibits
slightly reduced peaks over repeated laps. Further, the velocity plots in Fig.~\ref{fig:sim_mpc_somparison}(i)--(k) show that RK-MPC
follows the commanded velocities well over repeated laps. Both SRB-MPC and RK-MPC complete the full 10-lap rollout (about 250~s), whereas SE3-KMPC becomes unstable and fails within the first few seconds (about 5~s).

The stability limitation of SE3-KMPC is highlighted in
Fig.~\ref{fig:sim_mpc_somparison}(c)--(d) and (h). The roll/pitch phase portraits
in Fig.~\ref{fig:sim_mpc_somparison}(c)--(d) show large excursions for SE3-KMPC.
The zoomed view shows that RK-MPC remains bounded near the nominal region over
the same rollout. This behavior is consistent with the rotation-manifold drift
in Fig.~\ref{fig:sim_mpc_somparison}(h). During runtime we log the internal
predicted rotation $\hat R$ used to construct the SE(3)-aware lifted basis. We
observe that $\hat R$ progressively departs from ${SO}(3)$, as quantified
by increasing $\|\hat R^\top \hat R-I\|_F$ and $|\det(\hat R)-1|$. This drift can
be reduced by explicitly projecting $\hat R$ back onto ${SO}(3)$ at each
step, but doing so adds extra operations and implementation complexity.

Fig.~\ref{fig:sim_mpc_somparison}(e)--(f) reports RMSE distributions for
linear-velocity and angular-rate tracking across rollouts. For a compact summary,
we aggregate RMSE across axes by averaging the per-axis means. In linear velocity,
SRB-MPC and RK-MPC are essentially identical, with an overall mean RMSE of
$0.0199$~m/s (SRB-MPC) and $0.0201$~m/s (RK-MPC), while SE3-KMPC is substantially
worse at $0.0979$~m/s. In angular-rate tracking, RK-MPC achieves the lowest
overall mean RMSE of $0.0948$~rad/s, compared to $0.1748$~rad/s for SRB-MPC and
$0.7001$~rad/s for SE3-KMPC. This improvement is driven primarily by the yaw-rate
channel, where RK-MPC maintains a tight $\omega_z$ response across laps. 

Finally, we report the MPC solve time, i.e., the per-iteration wall-clock time required to solve the finite-horizon optimization and produce a new control sequence. In Fig.~\ref{fig:sim_mpc_somparison}(g), RK-MPC attains the lowest solve times with a mean of 0.89 ms, compared to the SE3-KMPC (1.40 ms) and the nonlinear SRB-MPC (1.57 ms).

Overall, these results show that RK-MPC preserves the command-tracking quality of
the nonlinear SRB-MPC baseline while offering a faster convex QP implementation.
In addition, RK-MPC avoids the rotation-consistency drift that destabilizes
SE3-KMPC, and therefore provides a more reliable Koopman-based MPC pipeline for
long-horizon locomotion tracking.

\section{Experimental Validation} \label{sec:experiments}
\begin{figure*}
  \centering
  \includegraphics[width=\linewidth]{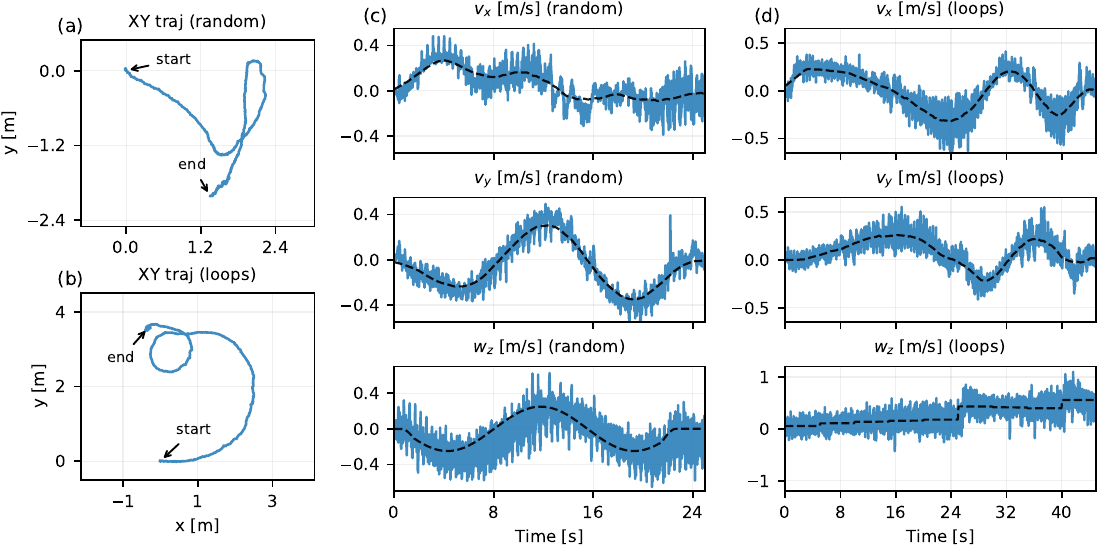}
  \caption{\textbf{Experiment 1: Velocity tracking on flat ground.}
  (a) Executed base $x$--$y$ trajectory during randomized planar commands (start/end annotated).
  (b) Executed base $x$--$y$ trajectory during smooth loop commands (start/end annotated).
  (c) Randomized segment: commanded (black dashed) and measured (solid) $(v_x,\,v_y,\,\omega_z)$.
  (d) Loop segment: commanded (black dashed) and measured (solid) $(v_x,\,v_y,\,\omega_z)$.}
  \label{fig:hw_mpc_tracking}
\end{figure*}

We validate the proposed Residual Koopman MPC (RK-MPC) on a Unitree Go1 in real-world hardware experiments. The full control stack runs fully onboard on an NVIDIA Jetson Xavier NX. At each control cycle, RK-MPC solves the convex QP in \eqref{eq:reskoop_mpc_ocp} using the residual-corrected predictor in \eqref{eq:mpc_corrected_clean}--\eqref{eq:mpc_lifted_clean}, and applies the first optimal contact-force command $u_k^\star$ in a receding-horizon fashion. User commands are specified as planar velocity setpoints $(v_x^{\mathrm{cmd}},\,v_y^{\mathrm{cmd}},\,\omega_z^{\mathrm{cmd}})$ and are converted by a smooth global planner into a world-frame SRB reference trajectory, yielding the reference sequence $\{x_{k+i}^\star\}_{i=0}^{N}$ used in the tracking objective \eqref{eq:reskoop_mpc_obj}. State estimation is performed using the same linear Kalman filter, which provides the SRB state $x_k$ (base pose and body velocities). Gait phase and contacts are generated by a wave-based gait scheduler, and swing footholds are planned using a Raibert-style velocity-tracking heuristic. Low-level swing and stance torques are computed by mapping Cartesian swing forces and MPC ground-reaction forces through the leg Jacobians as described in the hierarchical architecture of the Preliminaries (footstep planning and torque mapping in \eqref{eq:raibert_td}--\eqref{eq:tau_stance}). We evaluate in a blind-locomotion setting without exteroceptive sensing, relying only on proprioception and commanded references. The goal is to demonstrate accurate tracking of time-varying planar commands and robust locomotion under model mismatch and disturbances.

The state estimator and RK-MPC controller run at 500\,Hz, while the MPC model uses a discretization step of $0.01$\,s. Swing and stance joint torques are computed at 500\,Hz. On hardware, the controller supports agile planar tracking up to peak forward speeds of $0.6$~m/s under nominal flat-ground conditions.

\subsection{Experiment 1: Velocity tracking on flat ground}
\label{subsec:hw_tracking}
In this experiment, we evaluate the closed-loop tracking performance of the proposed RK-MPC on flat indoor terrain under time-varying planar velocity commands. The commanded references are specified as $(v_x^{\mathrm{cmd}},\,v_y^{\mathrm{cmd}},\,\omega_z^{\mathrm{cmd}})$, following the notation used in Sec.~\ref{subsec:data_collection}, and are converted by the planner into the reference sequence $\{x_{k+i}^\star\}_{i=0}^{N}$ used in \eqref{eq:reskoop_mpc_obj}. We consider two representative command profiles. First, we apply randomized planar commands (sampled and filtered as in Sec.~\ref{subsec:data_collection}) to excite frequent changes in direction and yaw rate. The resulting executed $x$--$y$ path is shown in Fig.~\ref{fig:hw_mpc_tracking}(a). Second, we apply a smoother, approximately sinusoidal command pattern that induces sustained curvature, resulting in a loop-like/circular trajectory in the plane (Fig.~\ref{fig:hw_mpc_tracking}(b)). 

Figure~\ref{fig:hw_mpc_tracking}(c)--(d) report the corresponding time histories of commanded (black dashed) and measured (solid) planar velocities $(v_x,\,v_y,\,\omega_z)$ for the randomized and loop segments, respectively. In both cases, the measured velocities closely track the commanded references, with bounded oscillations attributable to gait impacts and contact transitions. Overall, these results demonstrate that RK-MPC achieves reliable real-time planar command tracking on hardware in nominal flat-ground conditions.

\subsection{Experiment 2: Robustness to debris and external push}
\label{subsec:hw_debris_push}
We next evaluate the robustness of RK-MPC under contact disturbances and unmodeled terrain interactions in a blind-locomotion setting. First, we command the robot to traverse a debris field consisting of loose wooden pieces and small obstacles (Fig.~\ref{fig:hw_mpc_debris_push}(a)). Over $10$ trials at $v_x^{\mathrm{cmd}}=0.5\,\textrm{m/s},\,v_y^{\mathrm{cmd}}=0\,\textrm{m/s},\,\omega_z^{\mathrm{cmd}}=0\,\textrm{rad/s}$, the robot successfully completes $9/10$ runs without failure, demonstrating that the residual-corrected predictor maintains stable tracking despite intermittent foot slippage and height variations introduced by the debris. Figures~\ref{fig:hw_mpc_debris_push}(b)--(c) further summarize attitude stability by plotting phase portraits (roll versus $\omega_x$ and pitch versus $\omega_y$), comparing flat-ground walking against the debris condition. While debris induces larger excursions than flat terrain, the trajectories remain bounded and return to the near-origin region, indicating stable closed-loop behavior.

Second, we evaluate recovery from external disturbances by applying repeated lateral and longitudinal pushes during steady walking (Fig.~\ref{fig:hw_mpc_debris_push}(d)). We conduct two sequences of three pushes (six pushes total), and in all cases the robot maintains balance and returns to steady tracking without falling. Figures~\ref{fig:hw_mpc_debris_push}(e)--(f) show representative time histories of the measured planar velocities $(v_x,v_y)$ and base angles (roll, pitch) during the push trials, where each push produces a transient deviation followed by rapid recovery back toward nominal motion.

\begin{figure*}
  \centering
  \includegraphics[width=\linewidth]{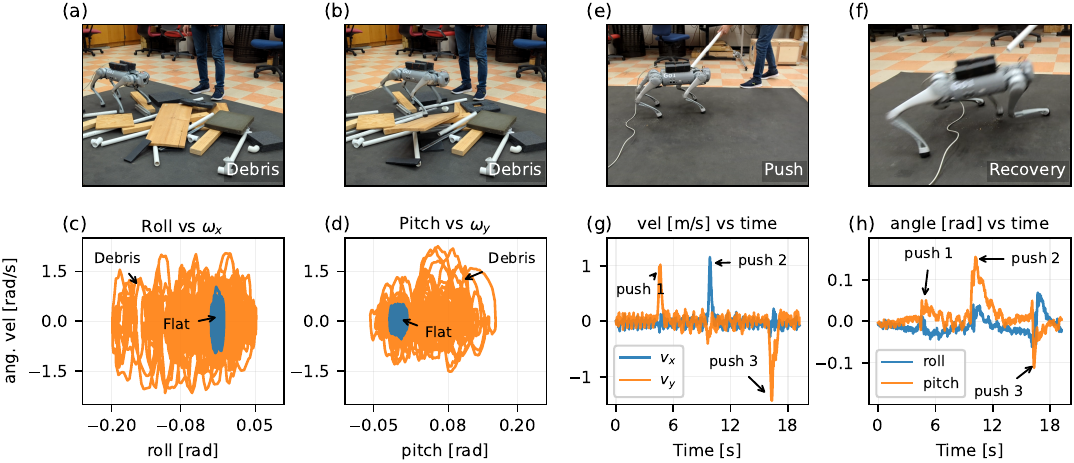}
  \caption{\textbf{Experiment 2: Debris traversal and push recovery.}
  (a) Representative snapshots of blind locomotion over a debris field.
  (b)--(c) Attitude phase portraits comparing debris (orange) and flat ground (blue): roll versus $\omega_x$ and pitch versus $\omega_y$.
  (d) External push trial with recovery during steady walking.
  (e) Measured planar velocities $(v_x,v_y)$ during the push sequence (push instants annotated).
  (f) Measured base angles (roll, pitch) during the push sequence (push instants annotated).}
  \label{fig:hw_mpc_debris_push}
\end{figure*}

\subsection{Experiment 3: Executing different gaits}
\label{subsec:hw_gaits}
We next evaluate whether the proposed RK-MPC framework generalizes across gait schedules beyond those used during model identification. Note that the residual Koopman predictor is trained using data collected under a trot gait, yet the MPC formulation in \eqref{eq:reskoop_mpc_ocp} depends on the gait only through the contact schedule $\sigma_k$ (which selects the active stance feet and corresponding constraints). To test gait generalization, we execute both a trot gait and a crawl gait on hardware using the gait timing parameters reported in Table~\ref{tab:gait_params}.

Figure~\ref{fig:hw_mpc_gaits} summarizes representative segments for each gait. The left column shows snapshots of the robot executing the gait patterns, and the bottom row visualizes the corresponding contact schedule over time. The right columns show the measured per-foot ground reaction forces for each gait (stacked as $f_x$, $f_y$, and $f_z$), illustrating consistent force modulation aligned with the stance phases. Despite being trained only with trot data, RK-MPC produces stable locomotion under the crawl schedule, indicating that the learned residual correction transfers across contact timing patterns and that the controller can accommodate different gaits through the scheduled constraints and foothold plan.

\begin{figure*}
  \centering
  \includegraphics[width=\linewidth]{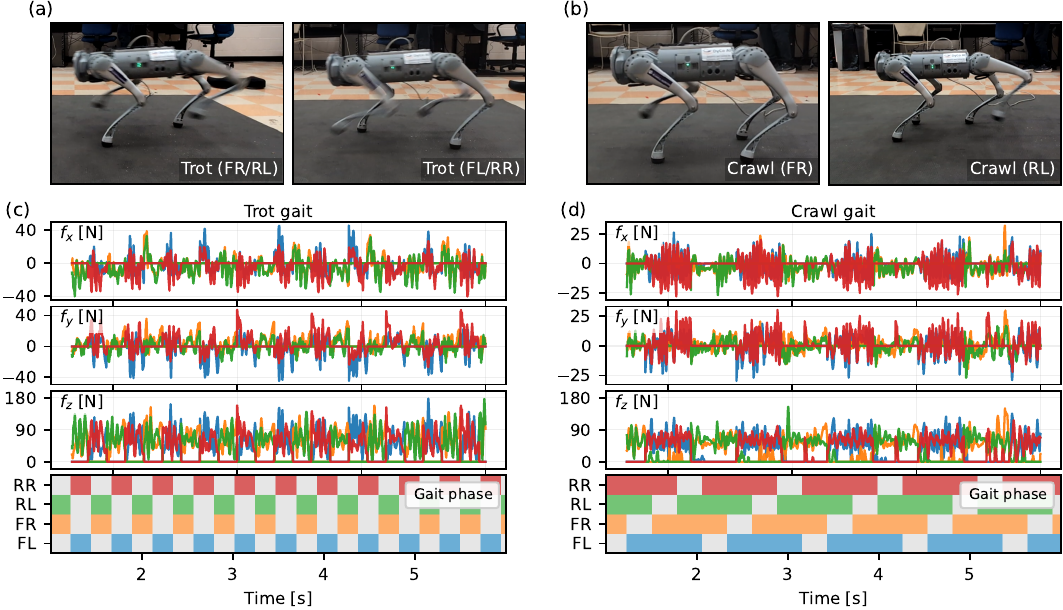}
  \caption{\textbf{Experiment 3: Gait generalization.}
  RK-MPC trained on trot data is executed on hardware under both trot and crawl gait schedules.
  representative snapshots (a-b). and measured per-foot ground reaction forces $f_x$, $f_y$, and $f_z$ for trot (c) and crawl (d).
  Bottom: corresponding gait phase/contact schedule for each leg (FR, FL, RR, RL), showing consistency between stance phases and force generation.}
  \label{fig:hw_mpc_gaits}
\end{figure*}

\subsection{Experiment 4: Generalization to off-road terrains}
\label{subsec:hw_offroad}
Finally, we evaluate the robustness of RK-MPC under unmodeled ground conditions by executing velocity-tracking trials on four representative outdoor terrains: grass, gravel, ice, and snow (Fig.~\ref{fig:hw_mpc_offroad}(a)--(d)). In all trials, the robot operates in a blind-locomotion setting without exteroceptive sensing and tracks commanded planar references specified by $(v_x^{\mathrm{cmd}},\,v_y^{\mathrm{cmd}},\,\omega_z^{\mathrm{cmd}})$.

On grass and gravel, we command primarily planar linear motion (no yaw-rate command) and observe stable locomotion with consistent tracking of $(v_x, v_y)$ over the duration of the trials (Fig.~\ref{fig:hw_mpc_offroad}(i)--(j)). The corresponding attitude phase portraits (Fig.~\ref{fig:hw_mpc_offroad}(e)--(f)) remain bounded and comparable to flat-ground behavior, with modest increases in roll/pitch excursions attributable to uneven contact and compliance.

On ice, the robot remains dynamically stable but begins to experience noticeable slip after several seconds. This behavior is consistent with the measured vertical contact forces saturating near the stance limit across legs (Fig.~\ref{fig:hw_mpc_offroad}(g)), indicating that the controller demands increased normal force to maintain traction while the available friction is insufficient. The contact schedule for this segment is shown in Fig.~\ref{fig:hw_mpc_offroad}(k).

On snow, we command a circular planar motion with a nonzero yaw-rate reference. Figure~\ref{fig:hw_mpc_offroad}(h) shows that the robot tracks $(v_x, v_y)$ smoothly, and Fig.~\ref{fig:hw_mpc_offroad}(l) shows close tracking of $\omega_z$ over the trial, demonstrating that RK-MPC can sustain coordinated turning motion on deformable terrain.

\begin{figure*}
  \centering
  \includegraphics[width=\linewidth]{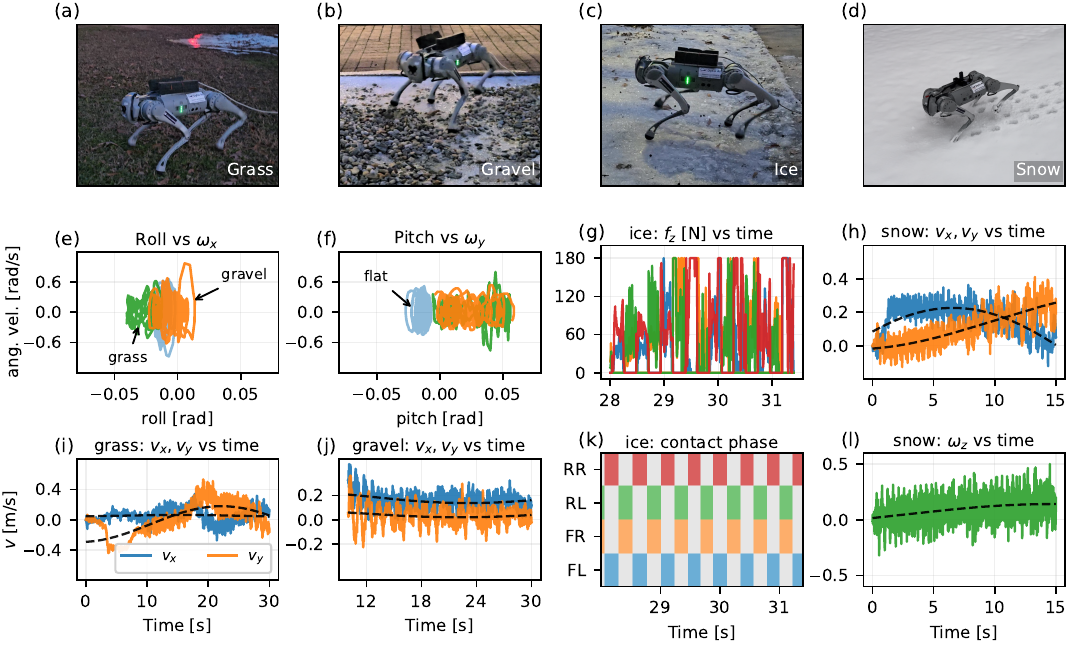}
  \caption{\textbf{Experiment 4: Off-road generalization.}
  (a)--(d) Representative snapshots of blind locomotion on grass, gravel, ice, and snow.
  (e)--(f) Attitude phase portraits (roll versus $\omega_x$ and pitch versus $\omega_y$) across terrains.
  (g) Ice: measured vertical ground-reaction forces $f_z$ approaching the stance limit during slip onset.
  (h) Snow: commanded (dashed) and measured (solid) planar velocities $(v_x, v_y)$ during a turning trial.
  (i)--(j) Grass and gravel: commanded (dashed) and measured (solid) $(v_x, v_y)$ time histories.
  (k) Ice contact schedule for the shown segment.
  (l) Snow: commanded (dashed) and measured (solid) yaw rate $\omega_z$.}
  \label{fig:hw_mpc_offroad}
\end{figure*}

\section{Discussion}
\label{sec:discussion}

The results show that directly learning a lifted linear predictor for the full
SRB state using a generic monomial EDMD dictionary can be unreliable for
contact-rich quadruped locomotion. Even when the one-step regression error is
reasonable, multi-step rollouts can become unstable, leading to rapidly growing
prediction error and divergence (Fig.~\ref{fig:residual_edmd_fit}(a) and
Table~\ref{tab:rmse_all_models}). This behavior is consistent with the fact that
high-degree polynomial lifts can amplify modeling and measurement errors, and
the resulting linear predictor does not inherently enforce physical structure or
boundedness. When embedded in a receding-horizon controller, such open-loop drift
corrupts the horizon predictions and degrades the quality of the optimized
contact forces, making monomial EDMD a poor drop-in replacement for a
physics-based template in MPC.

Using an $SE(3)$-inspired EDMD dictionary substantially improves prediction
accuracy relative to monomial EDMD and avoids catastrophic blow-up, yielding
lower RMSE in open-loop tests (Fig.~\ref{fig:residual_edmd_fit}(b,c) and
Table~\ref{tab:rmse_all_models}). However, this improvement comes with a distinct
failure mode in closed loop: the lift depends explicitly on the estimated
attitude through $\mathrm{vec}(R)$ and the rotation-aware features in
\eqref{eq:EDMD_SE3}. As a consequence, small attitude-estimation errors perturb
the lifted coordinates and compound over the MPC horizon, producing drift in the
internally propagated rotation variables and ultimately destabilizing SE3-KMPC in
long-horizon tasks (Fig.~\ref{fig:sim_mpc_somparison}(h)). This highlights a
practical trade-off: adding geometric structure can reduce model bias, but can
also increase sensitivity to state-estimation quality and the consistency of
$\hat R$ under repeated prediction.

The proposed RK-MPC avoids these issues by retaining a physics-based SRB template
inside the optimizer and learning only a compact residual correction applied to
the velocity channels. This preserves the scheduling and constraint structure of
standard convex SRB MPC (contact-set dependent constraints, planned moment arms,
and per-foot force decision variables), while using data to improve prediction
fidelity in regimes where the nominal template is imperfect. Across simulation
and hardware experiments, RK-MPC consistently provides reliable tracking and
stable locomotion without the lifted-state divergence of monomial EDMD or the
rotation-drift sensitivity observed with SE(3)-based lifts.

On hardware, RK-MPC further demonstrates practical real-time viability: the full
stack runs onboard on an NVIDIA Jetson Xavier NX with the estimator and
controller executing at $500$~Hz while the MPC discretization uses $\Delta t=0.01$~s
(Sec.~\ref{sec:experiments}). Beyond flat-ground tracking (Sec.~\ref{subsec:hw_tracking}),
RK-MPC maintains stability under disturbances and changing ground conditions.
In particular, the off-road study in Experiment~4 (Sec.~\ref{subsec:hw_offroad})
directly evaluates performance under terrain and friction effects (grass, gravel,
ice, and snow), showing that RK-MPC sustains velocity tracking and stable walking
across diverse outdoor surfaces.

Two extensions follow naturally from this work. First, RK-MPC can be made
adaptive by updating the residual model online using new transitions collected
during operation, enabling specialization to changing payloads, friction, or
actuator characteristics. Such updates could be performed through incremental
least-squares refinement of $(A^{\mathrm{res}},B^{\mathrm{res}},C^{\mathrm{res}})$ or periodic refitting of a lightweight
residual layer under safety constraints. Second, because RK-MPC remains a convex
QP, it provides a natural foundation for robust MPC designs with guarantees.
For example, uncertainty sets on the learned residual could be propagated through
the linear predictor to derive tube-based or constraint-tightening schemes,
retaining real-time solvability while improving robustness certification.

\section{Conclusion}
\label{sec:conclusion}

This paper presented a residual Koopman modeling and control framework for
quadruped locomotion that augments a standard SRB-based convex MPC with a compact
data-driven correction. By learning a linear residual predictor in lifted
coordinates and applying the correction only in the physical velocity channels,
RK-MPC preserves the structure and real-time solvability of convex contact-force
MPC while improving prediction fidelity under contact variability. In
simulation and Unitree Go1 hardware experiments, RK-MPC achieves reliable planar
velocity tracking, maintains stability under disturbances, and generalizes to
different gait schedules and off-road terrains, all while running fully onboard
at high control rates. Overall, these results suggest that residual Koopman
models offer a practical middle ground between purely physics-based templates and
fully lifted EDMD predictors for contact-rich legged locomotion.

\bibliographystyle{ieeetr}  
\bibliography{references}

\end{document}